\documentclass{article}

\usepackage[main, final]{neurips_2026}
\usepackage{amsmath}
\usepackage{soul}
\usepackage{graphicx}
\usepackage{subcaption}
\usepackage{natbib}
\usepackage[utf8]{inputenc} % allow utf-8 input
\usepackage[T1]{fontenc}    % use 8-bit T1 fonts
\usepackage{hyperref}       % hyperlinks
\usepackage{url}            % simple URL typesetting
\usepackage{booktabs}       % professional-quality tables
\usepackage{amsfonts}       % blackboard math symbols
\usepackage{amssymb}
\usepackage{amsthm}
\usepackage{nicefrac}       % compact symbols for 1/2, etc.
\usepackage{microtype}      % microtypography
\usepackage{xcolor}         % colors
\newtheorem{theorem}{Theorem}
\newtheorem{proposition}{Proposition}  % add this line
\newtheorem{remark}{Remark}

\usepackage{booktabs}
\usepackage{multirow}

\usepackage{overpic}

\definecolor{revcolor}{RGB}{0,70,190}
\newcommand{\rev}[1]{{\color{black}#1}}

\title{A General Kernel Framework for Non-CND Distance Measures Using $|\mathcal{D}|$-Dimensional Sparse Landmark Embeddings,}

\author{%
  Marcus M.~Noack \\
  Applied Mathematics and Computational Research Division, \\
  Lawrence Berkeley National Laboratory \\
  Berkeley, CA 94720, USA \\
  \texttt{MarcusNoack@lbl.gov} \\
   \And
  Maher B. Alghalayini \\
  Applied Mathematics and Computational Research Division, \\
  Lawrence Berkeley National Laboratory \\
  Berkeley, CA 94720, USA \\
   \AND
  Mark D. Risser \\
  Climate and Ecosystem Sciences Division, \\
  Lawrence Berkeley National Laboratory \\
  Berkeley, CA 94720, USA \\
}

\begin{document}

\maketitle

\begin{abstract}
Kernel methods, and Gaussian Processes (GPs) in particular, require a Hilbertian distance measure---one whose square is conditionally negative definite (CND)---to guarantee positive semi-definiteness (PSD) of the kernel matrix; a condition that fails for many natural input spaces, including smooth manifolds and spaces of probability distributions. We propose the Sparse Landmark Embedding (SLE) kernel, which eliminates this requirement entirely. Each input is embedded into a sparse feature vector via compactly supported bump functions centered at all $|\mathcal{D}|$ training points; applying any standard PSD kernel in this embedding space yields a kernel that is provably PSD for arbitrary distance measures. The compact support automatically controls embedding sparsity, keeping kernel matrices well-conditioned and computationally tractable despite the high ambient dimension. We provide theoretical guarantees on PSD, sparsity, stability, and universal approximation, and show that the SLE kernel matches or substantially exceeds domain-specific baselines in both predictive accuracy and uncertainty quantification using geodesic and Wasserstein distances.
\end{abstract}

\section{Introduction}
Modern machine learning applications increasingly require flexible, probabilistic models that can handle diverse data structures while quantifying uncertainty. Gaussian Process (GP) regression has emerged as a flexible kernel-based method for approximating unknown functions from limited observed data \citep{rasmussen2006gaussian}.
A GP defines a normal prior probability distribution $\mathcal{N}(\mathbf{m}, \mathbf{K})$ over an arbitrary set of function values 
$
\mathbf{f} = [f(x_1), f(x_2), \dots, f(x_N)]^T,
$
where $x \in \mathcal{X}$, with a mean function \(m(x)\) --- often assumed to be zero for simplicity --- and a covariance matrix
$
\mathbf{K} = \text{Cov}(\mathbf{f}, \mathbf{f}), \quad \mathbf{K} \in \mathbb{R}^{N \times N}.
$
The true underlying function generating the data is assumed to be \(f(x)\). The observed dataset 
$
\mathcal{D} = \{ (x_i, y_i) \}_{i=1}^{|\mathcal{D}|}
$
with cardinality $|\mathcal{D}|$ is assumed to result from the functional relationship
$
y_i = f(x_i) + \epsilon(x_i),
$
where the noise \(\epsilon(x_i)\) is drawn from a Gaussian distribution with zero mean. The Gaussian prior is commonly assumed to be defined over function values at the data points; that means $N=|\mathcal{D}|$. We denote the collection of all inputs by \(\boldsymbol{X}\) and the corresponding outputs by \(\mathbf{y}\). The covariance matrix is calculated by applying a positive semi-definite (PSD) kernel function to positional arguments; i.e., $\mathbf{K} = \text{Cov}(\mathbf{f}, \mathbf{f}) = [k(x_i, x_j)]_{i,j=1}^{N}$. Stationary kernels depend only on the distance between inputs; $k(x_i, x_j) = k(|x_i - x_j|)$; most non-stationary kernels also use some form of distance between input pairs in their formulation. Positive definiteness of such stationary and non-stationary kernels is guaranteed when the square of the underlying distance metric is conditionally negative definite (CND)---equivalently, that it be Hilbertian~\citep{berg1984harmonic}. This property is not generally satisfied for distance measures on non-Euclidean spaces.

This issue is particularly apparent when input data lie on manifolds (e.g., endowed with the geodesic distance), are probability distributions (Wasserstein distance), strings, trees, or graphs, many of which naturally admit distance measures that are not CND. For example, the Wasserstein distance $W_2$ between distributions is not CND in general~\citep{bachoc2021gaussian}, and geodesic distances on Riemannian manifolds similarly fail this property~\citep{feragen2015geodesic,haasdonk2007invariant}. As a result, naively applying kernels can yield indefinite Gram matrices, violating the mathematical requirements of GPs and other kernel methods.

A variety of workarounds have been developed. For instance, for smooth manifolds, intrinsic heat or diffusion kernels provide PSD alternatives at the expense of increased computational burden and the need for geometric information~\citep{lafon2004diffusion}. For distributions, the sliced Wasserstein distance~\citep{bonneel2015sliced} offers a computationally tractable and CND alternative, but may sacrifice accuracy and efficiency.

%%%landmark embedding kernels as a solution
One particularly interesting approach to handle non-CND distances without distorting geometry is to move from distance-based kernels to landmark-based embedding kernels, which map inputs into finite-dimensional feature spaces defined by distances to a selected set of reference points. Rather than relying on a single pairwise distance, these methods construct feature vectors of the form $\phi(x) = [d(x, \ell_1), \dots, d(x, \ell_m)]$, where $\{\ell_i\}_{i=1}^m$ are landmark points drawn from the data or placed strategically. Once embedded, a Euclidean-distance kernel --- such as any Mat\'ern kernel --- can be applied to these representations, ensuring positive semidefiniteness even when the original distance is not CND. This idea connects to early Nystr\"om and inducing-point approximations in kernel methods \citep{williams2001using, drineas2005nystrom, titsias2009variational, hensman2013gaussian}, as well as more recent landmark-based constructions for learning on manifolds and distributions \citep{jayasumana2013kernel}. Unlike sliced or projected distances, landmark embeddings preserve richer structural information, making them a well-suited candidate for extending Gaussian processes to non-Euclidean and non-CND settings.

%%% what is not going well with landmark embedding kernels
Despite their conceptual appeal, landmark-based kernels introduce significant practical challenges. The first difficulty lies in choosing or learning landmark locations. If landmarks are selected heuristically (e.g., via $k$-means or random subsampling), they may fail to capture important geometric or topological features of the data manifold, leading to degraded predictive performance \citep{Alaoui2015, musco2017recursive}. Conversely, jointly learning landmark positions as model parameters introduces a highly nonconvex optimization problem, in which gradients must propagate through distance computations and kernel evaluations, often resulting in unstable training dynamics. A second major challenge stems from the embedding's dimensionality. As the number of landmarks grows, each input is mapped to a feature vector in $\mathbb{R}^m$, where $m$ may be in the hundreds or thousands. While increasing $m$ improves geometric fidelity, it exacerbates the curse of dimensionality, leading to poor predictive performance and uncertainty quantification. Moreover, high-dimensional embeddings tend to yield poorly conditioned kernel matrices, which complicates both hyperparameter learning and posterior inference. Thus, practical deployment of landmark-based Gaussian processes requires a careful balance between expressivity (large $m$) and tractability (small $m$), a regime for which no principled selection mechanisms currently exist.

In this paper, we propose Sparse Landmark Embedding (SLE) kernels that treat all data points as landmarks --- therefore avoiding the need to select their positions --- and leverage bump-function-based embeddings for automatic dimensionality reduction. Our kernel operates on arbitrary distance measures, mitigates the need for CND distance approximations, and is computationally stable. \rev{An overview is given in Figure~\ref{fig:comp}. Our contributions are: (i) the SLE kernel construction itself, which renders the use of \emph{all} $|\mathcal{D}|$ training points as landmarks tractable and thereby eliminates the landmark-selection problem; (ii) theoretical guarantees on positive semi-definiteness, sparsity, conditioning, local geometric fidelity, and universal approximation (Theorems~\ref{thm:psd}--\ref{thm:scaling}, Proposition~\ref{prop:injective}); (iii) a non-stationary extension that preserves the PSD guarantee for arbitrary spatially varying bump parameters (Appendix~\ref{app:ns_kernel}); and (iv) an experimental evaluation against three domain-specific baselines on manifold- and distribution-valued GP regression, together with controlled ablation studies that isolate the effect of the bump embedding relative to raw-distance landmark embeddings (Appendix~\ref{app:ablation}).}

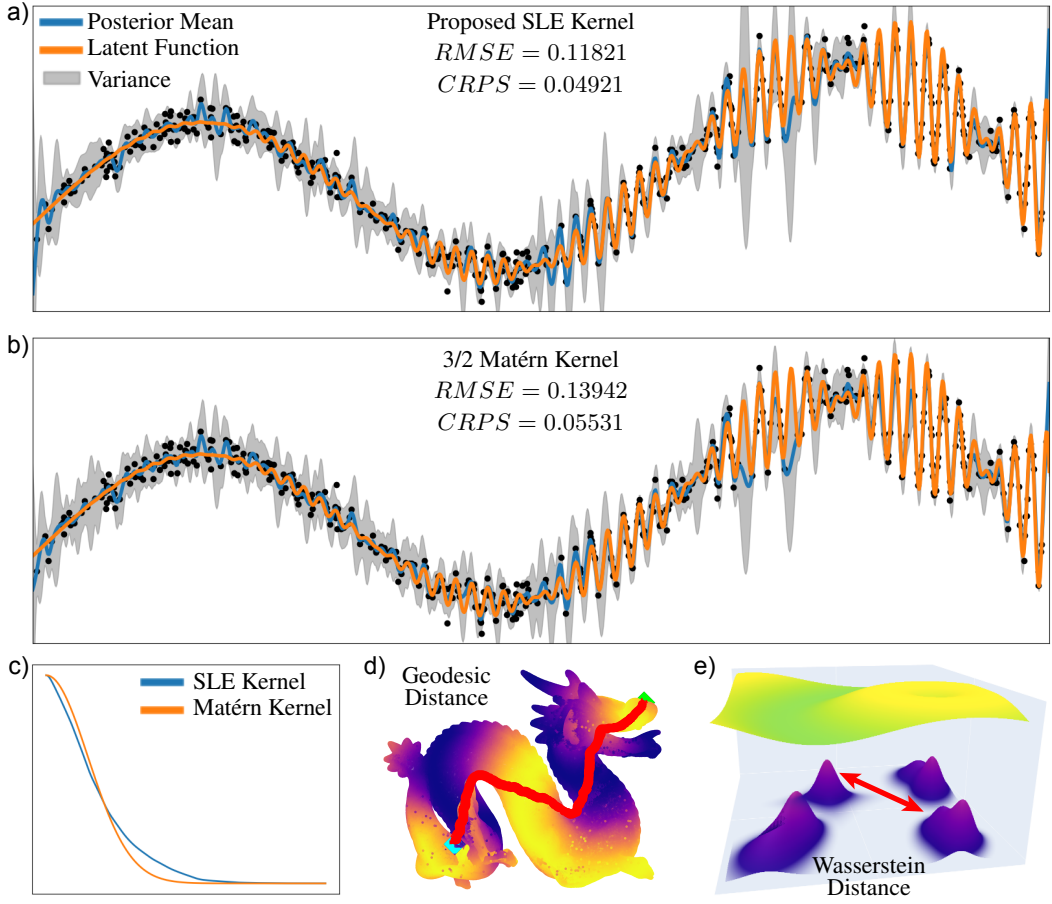
\begin{figure}[t]
\centering
% TODO (revision, Reviewer [5]): reduce the font size of the panel letters a)-e) in pics/Figure1.pdf (figure-file edit, not a LaTeX edit).
\begin{overpic}[width=\textwidth]{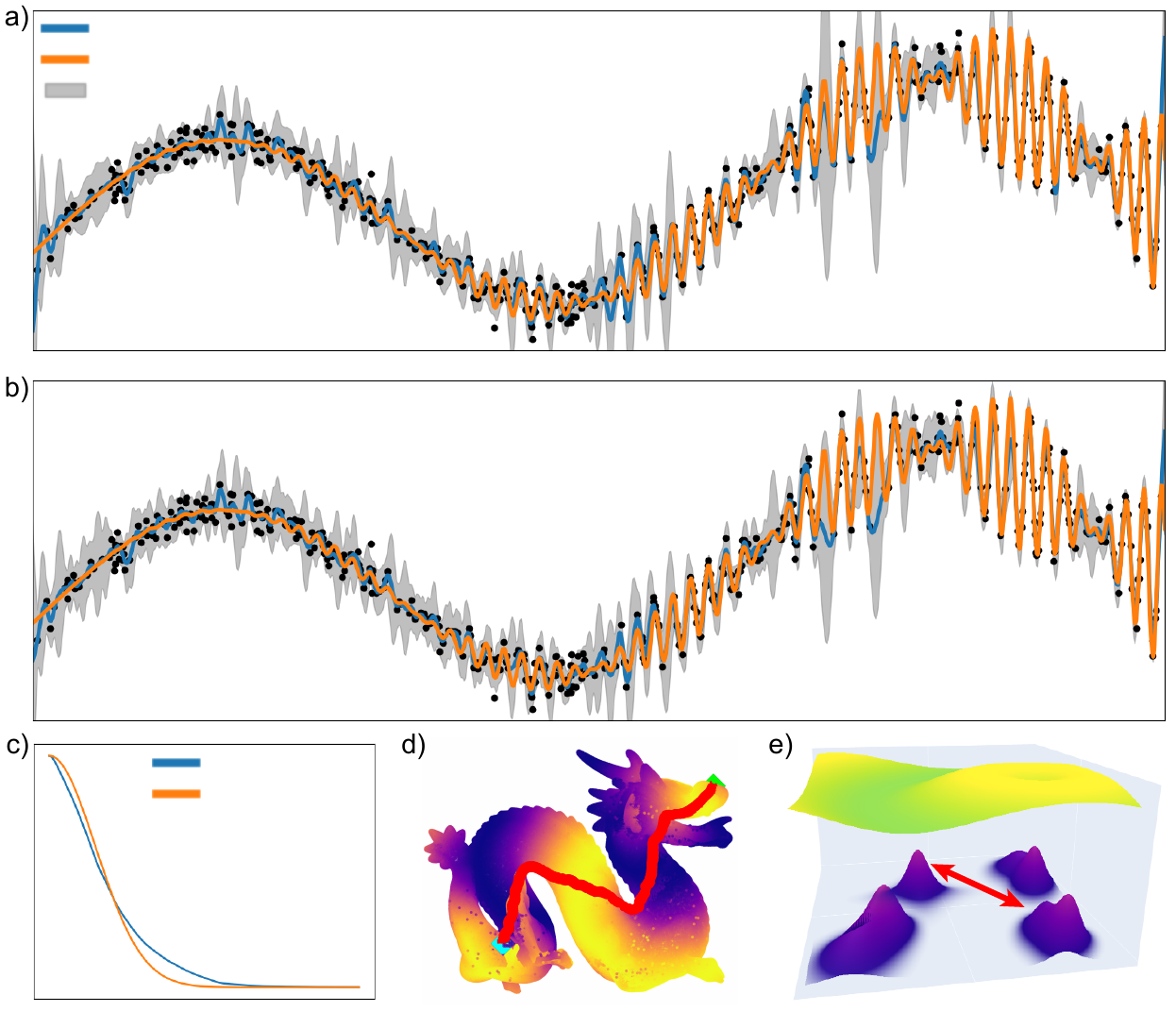}

    % ── Panel a: legend labels (right of existing swatches) ─────────────────
    \put(8, 83){\small Posterior Mean}
    \put(8, 80.5){\small Latent Function}
    \put(8, 77.5){\small Variance}

    % ── Panel a: title + metrics, centred in panel ──────────────────────────
    \put(50, 83){\makebox[0pt][c]{\small Proposed SLE Kernel}}
    \put(50, 80){\makebox[0pt][c]{\small $RMSE=0.11821$}}
    \put(50, 77){\makebox[0pt][c]{\small $CRPS=0.04921$}}

    % ── Panel b: title + metrics (moved down into panel b) ──────────────────
    \put(50, 51){\makebox[0pt][c]{\small 3/2 Mat\'ern Kernel}}
    \put(50, 48){\makebox[0pt][c]{\small $RMSE=0.13942$}}
    \put(50, 45){\makebox[0pt][c]{\small $CRPS=0.05531$}}

    % ── Panel c: labels placed where curves have flattened ──────────────────
    \put(18, 20.5){\small SLE Kernel}
    \put(18, 18){\small Mat\'ern Kernel}

     % ── Panel d: Geodesic Distance (top-left white space above dragon) ───────
    \put(42, 21){\makebox[0pt][c]{\small Geodesic}}
    \put(42, 19){\makebox[0pt][c]{\small Distance}}

    % ── Panel e: Wasserstein Distance (bottom white space below distributions)─
    \put(82, 3){\makebox[0pt][c]{\small Wasserstein}}
    \put(82, 1){\makebox[0pt][c]{\small Distance}}

\end{overpic}
\caption{We propose the Sparse Landmark Embedding kernel $k_{SLE}$, a general
(non-)stationary kernel for Gaussian Processes (GPs). Panels (a) and (b) show
a standard GP regression task performed with the proposed SLE kernel~(a) and a
Mat\'ern kernel ($\nu=3/2$)~(b), demonstrating the comparable behavior of the
two kernels in a standard regression scenario, with the proposed kernel
yielding a lower prediction error (RMSE) and better uncertainty quantification
(CRPS). In simple cases, our kernel structurally and empirically resembles
well-known stationary kernels such as the RBF and other Mat\'ern kernels~(c).
Unlike standard stationary kernels, the proposed SLE kernel
does not rely on a conditionally negative-definite (CND) distance metric,
making it applicable to input spaces that lack such a metric. This includes smooth
manifolds, where geodesic distances can be used without restriction~(d), as
well as sets of distributions, where exact Wasserstein distances can be
employed directly~(e).}
\label{fig:comp}
\end{figure}

%%%%%%%%%%%%%%%%%%%%%
\section{Related work} \label{sec:relatedwork}
A foundational requirement for kernel methods in general, and GPs in particular, is that the kernel be positive semi-definite (PSD), with distance-based kernels (e.g., RBF, Matérn) being PSD only when the associated distance metric is conditionally negative definite (CND)~\citep{berg1984harmonic}. When this condition is violated, applying standard stationary and non-stationary kernels can yield indefinite covariance matrices and unstable inference. This phenomenon has been studied in several disciplines. 

\textbf{Manifolds.} In the context of smooth Riemannian manifolds, the geodesic distance $d_\mathcal{M}(x, x')$ is not CND for general manifolds~\citep{feragen2015geodesic}. This prohibits the direct use of stationary radial kernels. Extrinsic approaches seek to embed the manifold in Euclidean space and use chordal distances, at the cost of geometric distortion. Intrinsically, heat kernels and spectral Laplace-Beltrami kernels leverage the underlying geometry to guarantee the PSD property, following from the spectral properties of the associated self-adjoint operators~\citep{borovitskiy2021matern, borovitskiy2020Matern}. However, these methods can be computationally demanding and require access to geometric features such as Laplacian eigenfunctions~\citep{lafon2004diffusion}.

\textbf{Distributions.} Probability distributions endowed with the Wasserstein metric face an analogous challenge: $W_2$ is not CND except in specific cases, for example, in one dimension~\citep{bachoc2021gaussian}. Consequently, kernels such as $\exp(-W_2^2(\mu,\nu)/\sigma^2)$ are indefinite in general. The sliced Wasserstein distance~\citep{bonneel2015sliced} improves the situation as it is CND, but can incur both computational cost and reduced accuracy. Other alternatives, such as entropically regularized Sinkhorn distances, partly restore the practical PSD property but are not guaranteed in all cases.

\textbf{Structured Data.} For data structured as sequences (strings) or trees, edit distances like Levenshtein and tree edit distance are common, but not CND. The resulting radial kernels are usually indefinite. To circumvent this, structured kernels that rely on feature vectors of n-grams, subsequences, or subtrees have been proposed, guaranteeing PSD via explicit inner-product constructions~\citep{lodhi2002text, chen2023fastkassim}.

\textbf{Graphs.} In graph domains, shortest-path distances are not generally CND, rendering direct radial kernels indefinite. PSD kernels such as diffusion, resistance, or commute-time kernels are constructed from the spectral or stochastic properties of graph Laplacians~\citep{von2008consistency, nikolentzos2021graph, ralaivola2005graph}, ensuring mathematical validity.

\textbf{Landmark/Nystr\"om Methods.} A prominent set of alternatives use landmarks or Nystr\"om methods, embedding each point as a feature vector of its distances to selected landmarks~\citep{williams2001using, drineas2005nystrom, jayasumana2013kernel}. Once embedded, a standard Euclidean kernel can be applied, preserving PSDness independently of the original metric's CND property. These ideas underlie scalable GP approximations like inducing points~\citep{titsias2009variational, hensman2013gaussian}. However, landmark-based embeddings inevitably involve critical choices about the number and placement of landmarks. Heuristic selections (e.g., $k$-means, subsampling) may not capture underlying geometric features, reducing predictive performance~\citep{Alaoui2015, musco2017recursive}, while learning landmarks involves challenging non-convex optimization. Embedding dimensionality also introduces a tradeoff between expressivity and tractability, complicated by the curse of dimensionality and associated numerical instabilities.

\citet{wu2018d2ke} propose D2KE, a general framework for constructing PSD
kernels from arbitrary dissimilarities by embedding inputs via distances to a
reference set and applying a standard kernel in the resulting feature space.
The SLE kernel is an instantiation of this framework distinguished by three
choices: (i) compactly supported bump functions in place of raw distances as
the embedding map, which induces automatic sparsity; (ii) all training points
as landmarks, eliminating landmark selection; and (iii) a GP-specific analysis
with guarantees on conditioning, stability, and universal approximation. The
first choice enables the second: a dense raw-distance embedding must keep its
dimension small to avoid ill-conditioning and distance concentration, whereas
compact support decouples the conditioning of $K$ from $|\mathcal{D}|$
(Theorems~\ref{thm:stability}--\ref{thm:concentration}). The ablation in
Appendix~\ref{app:ablation} is a controlled comparison against this
raw-distance embedding.

In summary, while a rich ecosystem of kernels and workarounds has extended GP regression beyond Euclidean domains, most require nontrivial geometric information, incur high computational cost, or sacrifice expressive fidelity. Our work contributes to this landscape by proposing a computationally efficient\rev{, provably PSD kernel that operates directly on the native distance measure of the input space --- without projections, slicing, or surrogate metrics (made precise in Proposition~\ref{prop:injective}) --- and is} suitable for arbitrary distances and learning in abstract non-Euclidean spaces.

\section{Background}\label{sec:background}

A major challenge in kernel design is ensuring the covariance function remains PSD 
when non-Euclidean or data-driven distances are used. Kernels built from a distance 
$d(x, x')$ whose square is not conditionally negative definite (CND) may yield indefinite 
Gram matrices, compromising both mathematical consistency and numerical stability of 
GP inference \citep{berg1984harmonic}. Hilbertian distance metrics guarantee, via Schoenberg's theorem \citep{berg1984harmonic}, 
that kernels of the form $k(x, x') = \exp(-d(x,x')^2/\sigma^2)$ are PSD.
A distance $d(x, x')$ is called \emph{Hilbertian} if $(X, d)$ can be isometrically 
embedded into a Hilbert space $\mathcal{H}$, i.e., there exists $\phi: X \to 
\mathcal{H}$ such that
$
    d(x, x') = \|\phi(x) - \phi(x')\|_{\mathcal{H}}.
$
Not all common distances satisfy this property: while the Euclidean
distance is Hilbertian, the Wasserstein-2 distance in dimensions greater than one 
is not \citep{gabriel2019computational}, and neither are geodesic, $l_1$ (in dimension $\geq2$), and other common distances, 
precluding their direct use in standard 
kernel methods. In this work, we propose a kernel construction that guarantees PSD 
even when the underlying distance is not CND.

\section{Methodology}\label{sec:methodology}

Our goal is to construct a kernel that (i) is provably positive semi-definite (PSD) 
for arbitrary, potentially non-CND distance measures, (ii) \rev{operates directly on 
the native distance $d$ of the input space, without replacing it by a projected, 
sliced, or otherwise distorted surrogate (made precise in 
Proposition~\ref{prop:injective})}, and (iii) remains 
computationally tractable. We build toward this construction in three steps: first 
establishing the core theoretical insight that motivates landmark embeddings, then 
identifying the practical obstacles of naive implementations, and finally showing 
how compactly supported bump functions resolve both obstacles simultaneously.

\subsection{The core insight: geometry-free PSD kernels via embeddings}

The fundamental observation underlying our approach is the following. Let 
$\mathcal{X}$ be any input space equipped with an arbitrary distance $d(\cdot, 
\cdot)$, not necessarily CND, and let $\phi : \mathcal{X} \rightarrow \mathbb{R}^m$ 
be any mapping into a Euclidean space. If \rev{$h$} $: \mathbb{R}^m \times \mathbb{R}^m 
\rightarrow \mathbb{R}$ is a PSD kernel on $\mathbb{R}^m$, then the composed kernel
\begin{equation}
    k(x, x') = \rev{h}(\phi(x), \phi(x'))
\end{equation}
is automatically PSD on $\mathcal{X}$, for any choice of $\phi$. This follows 
directly from the definition of positive semi-definiteness: for any finite set 
$\{x_1, \ldots, x_N\} \subset \mathcal{X}$ and any $\mathbf{c} \in \mathbb{R}^N$,
\begin{equation}
    \sum_{i=1}^N \sum_{j=1}^N c_i c_j k(x_i, x_j) 
    = \sum_{i=1}^N \sum_{j=1}^N c_i c_j \rev{h}(\phi(x_i), \phi(x_j)) \geq 0,
\end{equation}
since \rev{$h$} is PSD on $\mathbb{R}^m$ and $\{\phi(x_i)\}$ is simply a finite 
collection of points in $\mathbb{R}^m$. Crucially, this guarantee is entirely 
independent of the geometry of $\mathcal{X}$ and of whether $d(\cdot, \cdot)$ is 
CND. The geometry of $\mathcal{X}$ enters only through $\phi$, which can be 
constructed from $d$ in any way we choose.

\rev{\begin{remark}[Requirements on $d$]\label{rem:distreq}
No properties whatsoever are required of $d$ for the validity of the resulting kernel: Theorem~\ref{thm:psd} places no conditions on $d(\cdot,\cdot)$, since positive semi-definiteness is inherited entirely from the kernel applied in the embedding space. In particular, $d$ need not satisfy the triangle inequality, need not be symmetric, and may be noisy or inconsistent. Only auxiliary results require more of $d$ (Theorems~\ref{thm:smooth}, \ref{thm:sparsity}, \ref{thm:sparsity_scaling}).
\end{remark}}

This insight suggests a general strategy: encode the geometry of $(\mathcal{X}, d)$ 
into $\phi$, and then apply a standard Euclidean kernel \rev{$h$} in the embedding 
space. The remaining question is how to design $\phi$ so that it faithfully 
represents the geometry of $\mathcal{X}$ while keeping the kernel computationally 
tractable.

\subsection{Landmark embeddings and their limitations}\label{sec:limitations_landmark}

A natural choice for $\phi$ is a \emph{landmark embedding} \citep{gao2019gaussian, Scholkopf2002}: given a set of 
reference points $\{\ell_1, \ldots, \ell_m\} \subset \mathcal{X}$, embed each 
input by its distances to all landmarks,
\begin{equation}
    \phi(x) = [d(x, \ell_1), \ldots, d(x, \ell_m)]^\top \in \mathbb{R}^m.
\end{equation}
This construction is intuitive and general: it uses only the distance $d$, makes 
no assumptions about the geometry of $\mathcal{X}$, and produces a Euclidean 
feature vector to which any standard kernel can be applied. However, naive landmark 
embeddings face two fundamental and coupled difficulties.

\rev{\textbf{Landmark selection.} Heuristic selection (e.g., $k$-means, random
subsampling) may miss important geometric or topological features of the data,
degrading predictive performance \citep{Alaoui2015, musco2017recursive}, while
jointly optimizing landmark positions introduces a highly nonconvex problem with
gradients propagating through distance and kernel evaluations, often yielding
unstable training dynamics.

\textbf{Dimensionality.} Increasing the number of landmarks $m$ improves geometric
fidelity but subjects the embedding to the curse of dimensionality: pairwise
Euclidean distances concentrate, kernel matrices become poorly conditioned, and
both hyperparameter learning and posterior inference deteriorate
\citep{Alaoui2015}.}

These two problems are fundamentally coupled. Good geometric coverage of 
$\mathcal{X}$ requires many landmarks, but many landmarks produce high-dimensional, 
ill-conditioned embeddings. Any principled solution must address both simultaneously.

\subsection{Bump-function embeddings: resolving both problems at once}

We resolve both problems through a single design choice: replacing the raw distance embedding with a \emph{compactly supported bump-function embedding}, and using all $|\mathcal{D}|$ training points as landmarks. Note that $|\mathcal{D}|$ therefore serves simultaneously as the dataset cardinality and the embedding dimension — this is not a notational coincidence but a deliberate design choice that eliminates the need to separately specify or optimize the number of landmarks.

\rev{Throughout this work, a \emph{bump function} $b(\cdot)$ is any function of the
distance that is (i) compactly supported on $[0, r)$ for a radius $r>0$, (ii)
smooth on its support, and (iii) strictly positive (and strictly decreasing) on
its support. Our specific choice, defined in Eq.~\eqref{eq:bump} below and
visualized in Appendix~\ref{app:bump_vis}, is one member of this admissible
family; any other function with these properties may be substituted without
affecting the guarantees of this paper.}

The compact support of the bump functions \citep{noack2017hybrid} is the key mechanism. Because each bump 
function is exactly zero beyond a radius $r$ from its center, a given input $x$ 
will activate only the bump functions of nearby landmarks --- those within distance 
$r$. The embedding vector $\phi(x)$ is therefore \emph{sparse}: most of its $|\mathcal{D}|$ 
entries are exactly zero, with only a small number of nonzero entries corresponding 
to the local neighborhood of $x$. This sparsity has two immediate consequences.

First, it resolves the dimensionality problem. Although the ambient dimension of the 
embedding is $|\mathcal{D}|$, the effective dimension --- the number of nonzero entries --- 
remains small and controlled by the radius $r$, independently of $|\mathcal{D}|$. The curse 
of dimensionality is therefore avoided: pairwise distances in the embedding space 
remain informative, and kernel matrices remain well-conditioned as $|\mathcal{D}|$ grows 
(Theorems~\ref{thm:stability} and~\ref{thm:concentration}\rev{; empirically verified in Appendix~\ref{app:conditioning}}).

Second, it makes landmark selection trivial. Because the embedding is sparse, using 
all $|\mathcal{D}|$ training points as landmarks is computationally tractable. This choice 
guarantees maximal geometric coverage by construction, entirely bypassing the 
landmark selection problem and its associated nonconvex optimization.

\rev{Finally, the embedding is faithful to the native distance in the following
precise sense: no projection, slicing, or
surrogate metric is ever introduced --- the embedding is a function of the exact
native distance profile --- and the map from local distance profiles to embeddings
is injective.

\begin{proposition}[Local geometric fidelity]\label{prop:injective}
Fix landmarks $\{x_i\}_{i=1}^{|\mathcal{D}|}$ and bump parameters, and let
$b(\cdot;a,r_i,\beta)$ be strictly decreasing on its support $[0,r_i)$ for every
$i$. Then for any $x, x' \in \mathcal{X}$,
\[
\phi(x) = \phi(x')
\quad\Longleftrightarrow\quad
d(x, x_i) = d(x', x_i)\ \ \text{for every } i \text{ with } \min\bigl(d(x,x_i),\, d(x',x_i)\bigr) < r_i.
\]
That is, two inputs receive identical embeddings if and only if their distances to
all landmarks within reach agree exactly; the embedding discards only far-field
information (distances beyond the bump radii), which is a deliberate consequence of
locality, and distorts nothing within it. The proof is given in
Appendix~\ref{app:geo_proof}. An empirical distortion analysis comparing
embedding-space distances to native distances, for SLE versus the sliced
Wasserstein surrogate, is provided in Appendix~\ref{app:distortion}.
\end{proposition}}

\subsection{The Sparse Landmark Embedding (SLE) kernel}

Let $\{x_1, x_2, \ldots, x_{|\mathcal{D}|}\}$ denote the training data points and let 
$d(\cdot, \cdot)$ be a possibly non-CND distance metric, e.g., a geodesic distance 
on a manifold or the Wasserstein distance between distributions. We define the 
normalized bump function as
\begin{equation}
    b(d;\, a, r, \beta) = 
    \begin{cases} 
        a \exp\!\left( -\dfrac{\beta}{1 - d^2/r^2} + \beta \right) & \text{if } d < r, \\[6pt]
        0 & \text{otherwise,}
    \end{cases}
    \label{eq:bump}
\end{equation}
where $a > 0$ is the amplitude, $r > 0$ is the support radius, and $\beta > 0$ is 
a shape parameter controlling the flatness of the bump. This gives rise to the 
sparse landmark embedding
\begin{equation}
    \phi(x) = \bigl[b(d(x,x_1);\,a,r,\beta),\; b(d(x,x_2);\,a,r,\beta),\; 
    \ldots,\; b(d(x,x_{|\mathcal{D}|});\,a,r,\beta)\bigr]^\top \in \mathbb{R}^{|\mathcal{D}|}.
    \label{eq:embedding}
\end{equation}
Any stationary and non-stationary kernel can now be applied to the embedding space.
For example, applying the RBF kernel in the embedding space yields the particular SLE kernel:
\begin{equation}
    k_{\mathrm{SLE-RBF}}(x, x') = \sigma^2 \exp\!\left( -\frac{\|\phi(x) - \phi(x')\|^2}{2\ell^2} \right),
    \label{eq:sle}
\end{equation}
where $\sigma^2 > 0$ is the signal variance and $\ell > 0$ is the length scale. 
By the argument of Theorem \ref{thm:psd}, $k_{\mathrm{SLE}}$ is 
immediately PSD on $\mathcal{X}$ for any distance metric $d$. Any other standard kernel 
that is PSD on $\mathbb{R}^{|\mathcal{D}|}$ --- including the entire Mat\'{e}rn family --- can be 
used in place of the RBF kernel, yielding a corresponding SLE variant. 
\rev{We write SLE-RBF, SLE (Mat\'ern $\nu=3/2$), etc.\ to
indicate the inner kernel applied to the embedding, and simply SLE when the inner
kernel is clear from context; all experiments in Section~\ref{sec:experiments} use
Mat\'ern inner kernels, matched to the kernel order of the respective baseline.}
The SLE kernel maintains the original kernel's expressivity and universal approximation properties 
(Theorem \ref{thm:express}) and can be reduced in certain conditions to a stationary kernel in the original domain (Theorem \ref{thm:reduce}).
The SLE kernel's differentiability properties are inherited from the kernel applied to the embedding (Theorem \ref{thm:smooth}). See Theorem \ref{thm:scaling} for some notes on scaling properties of the kernel. Ablation studies demonstrating the effect of the bump function in the embedding are presented in Appendix \ref{app:ablation}.
\rev{All bump parameters, including the radius $r$, are hyperparameters learned by
marginal-likelihood maximization with data-adaptive bounds (see
Appendix~\ref{app:more_info_dragon}); $r$ thus plays the role of, and is selected
by the same mechanism as, a length scale in a standard stationary kernel.
Sensitivity analyses over $r$, the bump amplitude, and the choice of inner kernel
are reported in Appendix~\ref{app:sensitivity}.}

\rev{\paragraph{Non-stationary extension.}
Because the PSD guarantee of Theorem~\ref{thm:psd} holds for \emph{any} embedding
$\phi$, all bump parameters may vary freely as functions of position in
$\mathcal{X}$ without endangering validity. We develop this
extension, including the PSD proof and the roles of the individual parameter
fields, in Appendix~\ref{app:ns_kernel}.}

\section{Experiments and results}\label{sec:experiments}
We evaluate the proposed SLE kernel across three settings of increasing geometric complexity: 
two GP regression examples on smooth manifolds using geodesic distances, and one on sets of 
probability distributions using the Wasserstein-2 distance. In all experiments, 
predictive performance is assessed via root mean square error (RMSE), continuous ranked probability score (CRPS), and prediction interval coverage probability (PICP) at the nominal 95\% 
level, with CRPS and PICP serving as the primary indicators of uncertainty 
calibration quality. \rev{Throughout, lower is better for RMSE and CRPS, and closer
to the nominal 0.95 is better for PICP. The manifold benchmark problems (meshes,
target functions, and kernel orders) are taken directly from the baseline
publications \citep{borovitskiy2020Matern, Mostowsky2025} to preclude benchmark
selection in our favor. Sensitivity analyses for the bump radius, amplitude, and
inner kernel are reported in Appendix~\ref{app:sensitivity}, empirical
condition-number measurements in Appendix~\ref{app:conditioning}, and wall-clock
runtime comparisons in Appendix~\ref{app:cost}.}

\subsection{Dragon manifold}

We benchmark the proposed SLE kernel with geodesic distances against the
Riemannian Mat\'{e}rn kernel~\citep{borovitskiy2020Matern}, which is defined
via stochastic partial differential equations based on Laplace--Beltrami
eigenpairs. The Dragon mesh from the referenced work consists of 100,179
vertices used as GP input points, with output values defined as the sine of the
geodesic distance from the dragon's snout. Following the referenced work, the
data are assumed noiseless ($10^{-5}$ nugget) with a zero prior mean. Predictive performance was
evaluated across training sizes of 50--1000 points, with 30 randomly sampled
datasets per size. The Riemannian Mat\'{e}rn kernel was tested with 100, 500,
and 1000 eigenpairs. Table~\ref{tab:dragon_comparison} reports the mean and
standard error of RMSE, CRPS, and PICP for training sizes of 400, 600, and 800
points; the complete results appear in Figure~\ref{fig:dragon_comparison} in
Appendix \rev{\ref{app:manifold_results}}. The SLE kernel consistently outperformed the Riemannian Mat\'{e}rn
kernel across all training sizes and eigenpair configurations. More information is included in Appendix \ref{app:more_info_dragon}. 

\vspace{-0.5cm}
\begin{table}[ht]
\centering
\caption{Test RMSE, CRPS, and PICP (95\%) for Riemannian Kernel variants and SLE (Mat\'ern) 
for the Dragon example. Values reported as mean $\pm$ standard error. Dashes indicate 
unavailable values due to instability in the computation of the posterior covariance. Best performing method in bold. ``--'' is used for repeatedly unstable executions (see Appendix \ref{app:manifold_results}).}
\label{tab:dragon_comparison}
\begin{tabular}{llccc}
\toprule
Metric & Model & \multicolumn{3}{c}{Training Size} \\
\cmidrule(lr){3-5}
 & & 400 & 600 & 800 \\
\midrule
\multirow{4}{*}{RMSE \rev{($\downarrow$)}}
 & Riem. (100 eigenpairs)  & $0.171 \pm 0.006$ & $0.144 \pm 0.001$ & $0.137 \pm 0.001$ \\
 & Riem. (500 eigenpairs)  & $0.242 \pm 0.020$ & $4.257 \pm 0.580$ & $0.297 \pm 0.049$ \\
 & Riem. (1000 eigenpairs) & $0.109 \pm 0.002$ & $0.120 \pm 0.006$ & $0.473 \pm 0.073$ \\
 & SLE                 & $\mathbf{0.062} \pm \mathbf{0.002}$ & $\mathbf{0.042} \pm \mathbf{0.001}$ & $\mathbf{0.034} \pm \mathbf{0.001}$ \\
\midrule
\multirow{4}{*}{CRPS \rev{($\downarrow$)}}
 & Riem. (100 eigenpairs)  & $0.111 \pm 0.001$ & $0.101 \pm 0.000$ & $0.098 \pm 0.000$ \\
 & Riem. (500 eigenpairs)  & $0.100 \pm 0.018$ & $1.057 \pm 0.147$ & $0.090 \pm 0.006$ \\
 & Riem. (1000 eigenpairs) & --                & --                & --                \\
 & SLE                 & $\mathbf{0.030} \pm \mathbf{0.001}$ & $\mathbf{0.020} \pm \mathbf{0.001}$ & $\mathbf{0.015} \pm \mathbf{0.000}$ \\
\midrule
\multirow{4}{*}{\shortstack[l]{PICP (95\%)\\[2pt] \rev{($\to 0.95$)}}}
 & Riem. (100 eigenpairs)  & $0.000 \pm 0.000$ & $0.000 \pm 0.000$ & $0.000 \pm 0.000$ \\
 & Riem. (500 eigenpairs)  & $0.780 \pm 0.007$ & $0.000 \pm 0.000$ & $0.000 \pm 0.000$ \\
 & Riem. (1000 eigenpairs) & $\mathbf{0.912} \pm \mathbf{0.004}$ & $0.869 \pm 0.007$ & $0.609 \pm 0.025$ \\
 & SLE                 & $0.911 \pm 0.018$ & $\mathbf{0.944} \pm \mathbf{0.007}$ & $\mathbf{0.937} \pm \mathbf{0.005}$ \\
\bottomrule
\end{tabular}
\end{table}

\vspace{-0.5cm}
\subsection{Teddy Bear manifold}

We further benchmark the SLE kernel against the Geometric
kernel~\citep{Mostowsky2025}, a more recent manifold kernel. The Teddy Bear
mesh is reproduced from the referenced work and consists of 1,598 vertices used
as GP input points, with output values defined as a random sample from the
prior reported therein. Predictive performance was evaluated across training
sizes of 50--800 points, with 30 randomly sampled datasets per size.
Table~\ref{tab:teddy_comparison} reports the mean and standard error of RMSE,
CRPS, and PICP for training sizes of 200, 300, and 400 points; the complete
results appear in Figure~\ref{fig:teddy_comparison} in Appendix \rev{\ref{app:manifold_results}}.
Both kernels achieve comparable RMSE across all training sizes, indicating
similar posterior mean accuracy. However, the two kernels differ substantially
in uncertainty quantification. The SLE kernel consistently achieves lower CRPS
and maintains a PICP near the nominal 95\% level across all training sizes,
indicating well-calibrated predictive uncertainty. The Geometric kernel produced PICP values between 8\% and 27\%, reflecting severe overconfidence in which the 95\% prediction intervals capture only a small fraction of the true test values. These results were obtained using the reference implementation of \citep{Mostowsky2025} without modification, confirming that the result reflects the behavior of the published method rather than an implementation artifact. \rev{We stress that, in contrast to the spectral-truncation instability observed on the Dragon manifold, the Geometric kernel is numerically \emph{stable} in this experiment and matches SLE in RMSE; the reported PICP therefore reflects the published method's calibration behavior in its stable operating regime, not an instability artifact.} More information is discussed in Appendix \ref{app:additional_info_on_teddy}\rev{.}

\begin{table}[ht]
\centering
\caption{Test RMSE, CRPS, and PICP (95\%) for Geometric Kernel and SLE models for the 
Teddy Bear example. Values reported as mean $\pm$ standard error. Best performing method in bold.}
\label{tab:teddy_comparison}
\begin{tabular}{llccc}
\toprule
Metric & Model & \multicolumn{3}{c}{Training Size} \\
\cmidrule(lr){3-5}
 & & 200 & 300 & 400 \\
\midrule
\multirow{2}{*}{RMSE \rev{($\downarrow$)}}
 & Geometric Kernel & $42.147 \pm 0.316$ & $36.965 \pm 0.327$ & $34.094 \pm 0.338$ \\
 & SLE              & $\mathbf{41.825} \pm \mathbf{0.485}$ & $\mathbf{36.702} \pm \mathbf{0.558}$ & $\mathbf{34.976} \pm \mathbf{0.531}$ \\
\midrule
\multirow{2}{*}{CRPS \rev{($\downarrow$)}}
 & Geometric Kernel & $23.781 \pm 0.262$ & $18.805 \pm 0.178$ & $16.142 \pm 0.153$ \\
 & SLE              & $\mathbf{18.134} \pm \mathbf{0.197}$ & $\mathbf{14.672} \pm \mathbf{0.183}$ & $\mathbf{13.041} \pm \mathbf{0.130}$ \\
\midrule
\multirow{2}{*}{\shortstack[l]{PICP (95\%)\\[2pt] \rev{($\to 0.95$)}}}
 & Geometric Kernel & $0.100 \pm 0.003$ & $0.119 \pm 0.003$ & $0.138 \pm 0.003$ \\
 & SLE              & $\mathbf{0.932} \pm \mathbf{0.004}$ & $\mathbf{0.941} \pm \mathbf{0.003}$ & $\mathbf{0.929} \pm \mathbf{0.003}$ \\
\bottomrule
\end{tabular}
\vspace{-0.5cm}
\end{table}

\subsection{X-ray scattering data disguised as distributions}

We evaluate the SLE kernel on a dataset of 500 synthetic small-angle X-ray
scattering (SAXS) images designed to mimic real-world SAXS patterns from oriented soft-matter thin films, such as
block copolymer and liquid crystal systems, measured at synchrotron beamlines (Appendix \rev{\ref{app:additional_info_saxs}} Figure \ref{fig:saxs}). Each $64\times64$ image represents the 2D reciprocal-space intensity pattern of
a multi-domain lamellar sample, consisting of a fundamental scattering arc at
wavevector $q^*$ and a second harmonic at $2q^*$. The key structural parameter
is the inter-harmonic coupling disorder $\sigma_\text{coup}$, which controls the
degree to which the two arcs within each domain remain collinear. As
$\sigma_\text{coup}$ increases, the harmonic arcs decohere, reducing the
effective Young's modulus of the material along the measurement axis --- the
quantity used as the GP output $y$. The SLE kernel was applied with the full Wasserstein distance and benchmarked against the $\nu=3/2$ Mat\'{e}rn kernel with the sliced Wasserstein distance
across training sizes of 150, 200, and 250 images, with a fixed test set of 100
held-out images and 30 random dataset draws per configuration. As reported in Table~\ref{tab:dist_comparison}, the SLE kernel achieved comparable or slightly better RMSE and CRPS, with small differences across all training sizes. The key finding is not superiority but competitiveness: the SLE kernel matches a domain-adapted baseline that uses the sliced Wasserstein approximation, while operating directly on the true $W_2$ distance without any geometric preprocessing.
\rev{To decompose the contribution of the bump embedding from that of the exact
$W_2$ distance, we additionally evaluate the SLE kernel applied on top of the
\emph{sliced} Wasserstein distance (SLE -- Sliced Wass in
Table~\ref{tab:dist_comparison}). The SLE–Sliced Wasserstein variant achieves calibration comparable to the other two methods, with PICP within a few points of nominal at every training size, indicating that the calibration benefit of the bump construction persists regardless of the underlying distance. However, RMSE and CRPS for SLE–Sliced Wasserstein are higher than for both SLE–Wass and Matérn–Sliced Wass across all training sizes, suggesting that discarding far-field information via the bump radius compounds with the geometric distortion introduced by slicing. This indicates that the accuracy gains of the SLE kernel derive primarily from operating on exact distances rather than from the bump embedding alone, while the calibration gains are attributable to the sparse, local structure of the embedding itself, independent of the base distance to which it is applied.}

\vspace{-2mm}
\begin{table}[ht]
\centering
\caption{Test RMSE, CRPS, and PICP (95\%) for SLE (Mat\'ern) and Mat\'ern models for the 
\rev{SAXS} data. Values are reported as mean $\pm$ standard error of 30 random trials. Best-performing method in bold.}
\label{tab:dist_comparison}
\begin{tabular}{llccc}
\toprule
Metric & Model & \multicolumn{3}{c}{Training Size} \\
\cmidrule(lr){3-5}
 & & 150 & 200 & 250 \\
\midrule
\multirow{3}{*}{RMSE \rev{($\downarrow$)}}
 & Mat\'ern - Sliced Wass & $0.489 \pm 0.010$ & $0.445 \pm 0.011$ & $0.358 \pm 0.009$ \\
 & SLE - Wass           & $\mathbf{0.467} \pm \mathbf{0.010}$ & $\mathbf{0.396} \pm \mathbf{0.009}$ & $\mathbf{0.334} \pm \mathbf{0.007}$ \\
 & \rev{SLE - Sliced Wass} & $0.672 \pm 0.023$ & $0.575 \pm 0.014$ & $0.497 \pm 0.011$\\
\midrule
\multirow{3}{*}{CRPS \rev{($\downarrow$)}}
 & Mat\'ern - Sliced Wass & $0.221 \pm 0.004$ & $0.188 \pm 0.004$ & $0.148 \pm 0.005$ \\
 & SLE - Wass           & $\mathbf{0.210} \pm \mathbf{0.004}$ & $\mathbf{0.174} \pm \mathbf{0.003}$ & $\mathbf{0.146} \pm \mathbf{0.003}$ \\
 & \rev{SLE - Sliced Wass} & $0.288 \pm 0.006$ & $0.256 \pm 0.006$ & $0.220 \pm 0.006$ \\
\midrule
\multirow{3}{*}{\shortstack[l]{PICP (95\%)\\[2pt] \rev{($\to 0.95$)}}}
 & Mat\'ern - Sliced Wass & $\mathbf{0.940} \pm \mathbf{0.004}$ & $\mathbf{0.945} \pm \mathbf{0.004}$ & $0.963 \pm 0.004$ \\
 & SLE - Wass           & $0.934 \pm 0.006$ & $\mathbf{0.945} \pm \mathbf{0.004}$ & $\mathbf{0.959} \pm \mathbf{0.004}$ \\
 & \rev{SLE - Sliced Wass} & $0.922 \pm 0.004$ & $0.933 \pm 0.003$ & $0.937 \pm 0.004$ \\
\bottomrule
\end{tabular}
\vspace{-0.5cm}
\end{table}

\section{Discussion and conclusion}
\label{sec:discussion}

We propose the Sparse Landmark Embedding (SLE) kernel, a general-purpose kernel 
for Gaussian process regression that operates on arbitrary distance measures, including those that are not conditionally negative definite.
Embedding inputs through
compactly supported bump functions centered at all training points yields a
kernel that is provably PSD for any input geometry, removes the
landmark-selection problem, and mitigates the curse of dimensionality through
automatic sparsity, with guarantees on PSD, sparsity scaling, stability,
universal approximation, and the standard-kernel stationary limit.

Across the Dragon and Teddy Bear manifolds and the synthetic SAXS dataset, the
SLE kernel matched or improved on the RMSE of three domain-specific baselines
at every training size (Tables~\ref{tab:dragon_comparison},
\ref{tab:teddy_comparison}, \ref{tab:dist_comparison}). The difference in
uncertainty quantification is larger: CRPS and PICP were better calibrated in
every setting. Two structural properties account for this. The kernel operates
on the native distance of the input space --- geodesics on manifolds, exact
$W_2$ between distributions --- so no surrogate metric is introduced
(Proposition~\ref{prop:injective}); and the sparse embedding keeps Gram
matrices well conditioned (Theorem~\ref{thm:stability},
Appendix~\ref{app:conditioning}), avoiding the variance underestimation that
ill-conditioning produces.

The Riemannian Mat\'{e}rn kernel
\citep{borovitskiy2020Matern} requires Laplace--Beltrami eigenpairs and a
truncation level; performance deteriorates sharply once the training size
approaches that level, and away from the deterioration, SLE, which has no such
parameter, still leads on all three metrics. The Geometric kernel
\citep{Mostowsky2025} is numerically stable and matches SLE in point accuracy; 
severe overconfidence (PICP of 8--27\% at nominal 95\%)  reflects the calibration of the published method under
its own reference implementation. The sliced Wasserstein Mat\'{e}rn kernel
\citep{bachoc2021gaussian} behaves well throughout, and there we claim
competitiveness and a slight advantage; SLE operates on the true $W_2$ distance while the
baseline uses a sliced approximation whose distortion can degrade accuracy and
calibration for high-dimensional or multimodal distributions
\citep{nadjahi2019asymptotic}.

\textbf{Limitations.}
Forming the distance matrix between all test points and all $|D|$ landmarks
requires $O(N \cdot |D|)$ distance computations, even though kernel evaluations
are cheap (Theorem~\ref{thm:scaling}); cover trees, vantage-point trees, and
approximate nearest-neighbor methods need only the distance function and reduce
this to range queries, and the cost is shared by every distance-based
competitor considered here (Appendix~\ref{app:cost}). Predictive performance is
sensitive to the bump radii $r(x_i)$, and principled data-driven selection
remains open; we learn a global $r$ by marginal-likelihood maximization with
data-adaptive bounds, and Appendix~\ref{app:sensitivity} indicates a broad
well-performing region around the selected value. Locality carries a price:
distance information beyond the radii is discarded by construction
(Proposition~\ref{prop:injective} covers local profiles only), so problems
governed by global geometric structure require larger radii and less sparsity.

\section{Acknowledgements and disclosures}
\paragraph{Author contributions.} M.M.N.: Ideation, Kernel derivation, Performance comparisons, Software development, Manuscript; M.D.R.: Ideation, Kernel derivation, Manuscript; M.B.A.: Kernel derivation, Performance comparisons, Data curation, Test executions, Manuscript.

\paragraph{Funding}
This work was supported by the Center for Advanced Mathematics for Energy Research Applications (CAMERA), which is jointly funded by the Advanced Scientific Computing Research (ASCR) and Basic Energy Sciences (BES) within the U.S. Department of Energy Office of Science (DOE SC); by the U.S. DOE SC ASCR Applied Mathematics and Competitive Portfolios programs; and by the U.S. DOE SC Early Career Research program; all under Contract No. DE-AC02-05CH11231.

\paragraph{Data and code availability statement.}
The code and all data are available at \url{https://doi.org/10.5281/zenodo.22884130}.

\paragraph{Competing interests.}
The authors declare no conflicts of interest.

\newpage
\bibliographystyle{plainnat}
\bibliography{literature}

\newpage
\appendix
\section{Theoretical Properties and Proofs}

In this section, we present the properties of the proposed kernel, focusing on positive semi-definiteness, automatic dimensionality reduction, expressivity, stability, high-dimensional effects, connection to stationary kernels, and smoothness.

\subsection{Positive Semi-Definiteness (PSD) of the Kernel}

\begin{theorem}\label{thm:psd}
Let
\[
k(x, x') = \sigma^2 \exp\left( -\frac{\|\varphi(x) - \varphi(x')\|^2}{2\ell^2} \right)
\]
where $\varphi : \mathcal{X} \to \mathbb{R}^m$ is any mapping, and $\|\cdot\|$ is the standard Euclidean norm. Then $k$ is a positive semi-definite (PSD) kernel on $\mathcal{X}$.
\end{theorem}

\begin{proof}
Let $\{x_1, \ldots, x_N\} \subset \mathcal{X}$ be any finite collection of points, and let $z^{(i)} = \varphi(x_i) \in \mathbb{R}^m$ for $i=1,\ldots,N$. Consider the Gram matrix with entries
\[
K_{ij} = k(x_i, x_j) = \sigma^2 \exp\left(-\frac{\|z^{(i)} - z^{(j)}\|^2}{2\ell^2}\right).
\]
The function $(z, z') \mapsto \exp\left(-\frac{\|z-z'\|^2}{2\ell^2}\right)$ defines a positive semi-definite kernel on $\mathbb{R}^m$, as follows from Bochner's theorem since it is the Fourier transform (characteristic function) of a finite Gaussian measure. Therefore, for any real vector $\mathbf{c} \in \mathbb{R}^N$,
\[
\sum_{i=1}^N \sum_{j=1}^N c_i c_j K_{ij} \geq 0.
\]
Hence, $k$ is a positive semi-definite kernel for any choice of mapping $\varphi$.
\end{proof}

\rev{\subsection{Local Geometric Fidelity: Proof of Proposition~\ref{prop:injective}}
\label{app:geo_proof}

\begin{proof}[Proof of Proposition~\ref{prop:injective}]
Fix $i$ and write $b_i(\cdot) = b(\cdot;\,a, r_i, \beta)$, which by assumption is
strictly decreasing --- hence injective --- on its support $[0, r_i)$, and
identically zero on $[r_i, \infty)$. Consider the $i$-th embedding coordinates
$\phi_i(x) = b_i(d(x, x_i))$ and $\phi_i(x') = b_i(d(x', x_i))$.

($\Leftarrow$) If $\min(d(x,x_i), d(x',x_i)) \geq r_i$, then both coordinates are
zero and agree. If $\min(d(x,x_i), d(x',x_i)) < r_i$ and $d(x, x_i) = d(x', x_i)$,
the coordinates agree trivially. Hence the stated distance condition implies
$\phi(x) = \phi(x')$.

($\Rightarrow$) Suppose $\phi_i(x) = \phi_i(x')$ and
$\min(d(x,x_i), d(x',x_i)) < r_i$; without loss of generality $d(x, x_i) < r_i$,
so $\phi_i(x) = b_i(d(x,x_i)) > 0$ by strict positivity on the support. Then
$\phi_i(x') > 0$ as well, forcing $d(x', x_i) < r_i$, and injectivity of $b_i$ on
$[0, r_i)$ yields $d(x, x_i) = d(x', x_i)$. Applying this to every coordinate $i$
gives the claim.

Consequently, the embedding is an injective function of the local distance profile
$\{d(x, x_i) : d(x, x_i) < r_i\}$: within the reach of the bumps, the exact native
distances are encoded without projection or surrogate, and only far-field
information (distances beyond the radii) is discarded.
\end{proof}}

\subsection{Automatic Dimensionality Reduction and Sparsity}

\begin{theorem}\label{thm:sparsity}
Let $\mathcal{X}$ be any metric space equipped with a (not necessarily CND) distance $d(\cdot,\cdot)$. Consider a collection of $m$ landmark points $L = \{x_1, \ldots, x_m\} \subset \mathcal{X}$, and for each $i$ let $r_i > 0$ denote the radius of the compactly supported bump function centered at $x_i$. For any $x \in \mathcal{X}$, define the embedding
\[
\varphi(x) = \left[ b(d(x, x_1); a_1, r_1, \beta_1),\; \ldots,\; b(d(x, x_m); a_m, r_m, \beta_m) \right]^\top \in \mathbb{R}^m,
\]
where $b(d; a, r, \beta)$ is supported on $[0, r)$ (that is, $b(d; a, r, \beta) = 0$ if $d \geq r$).

Then at most $|\{i : d(x, x_i) < r_i\}|$ elements of $\varphi(x)$ are nonzero; all others are exactly zero. In particular, if the radii $\{r_i\}$ are small compared to the spacing of the points in $L$, and $x$ is randomly sampled according to some probability measure $\mu$ on $\mathcal{X}$, then the expected number of nonzero entries is
\[
\mathbb{E}_{x \sim \mu} \left[\, \| \varphi(x) \|_0 \, \right] = \sum_{i=1}^m \mathbb{P}_{x \sim \mu}\left[\, d(x, x_i) < r_i \,\right].
\]
If all radii $r_i = r$ and $\mu$ is sufficiently distributed across the domain then
\[
\mathbb{E}_{x \sim \mu}\left[\, \| \varphi(x) \|_0 \,\right] = m \cdot p_r \quad \text{where} \quad p_r := \mathbb{P}_{x \sim \mu}[\, d(x, x_i) < r \,] \ll 1,
\]
so the embedding is \emph{sparse}: as $m$ increases and $r$ is fixed, this expected count can be kept small relative to $m$.
\end{theorem}

\begin{proof}
The bump function $b(d(x, x_i); a_i, r_i, \beta_i)$ is nonzero if and only if $d(x,x_i) < r_i$; otherwise, it is zero by definition. Thus, in the embedding vector $\varphi(x)$, the $i$-th coordinate is zero unless $x$ lies within radius $r_i$ of landmark $x_i$. The set of indices with nonzero entries is thus $S_x = \{i : d(x, x_i) < r_i\}$, so $\|\varphi(x)\|_0 = |S_x|$. Averaging over $x$ drawn from $\mu$ yields the expected sparsity as stated. If, for each $i$, $p_i := \mathbb{P}_{x \sim \mu}[ d(x, x_i) < r_i ]$ is small (e.g., because $r_i$ is much less than the typical inter-landmark spacing), then the expected number of nonzero coordinates is $\sum_{i=1}^m p_i$, which can be made much less than $m$ by suitable choice of $r_i$. In the case where all radii are equal, this simplifies as above.
\end{proof}

\subsection{Expressivity and Universal Approximation}

\begin{theorem}\label{thm:express}
Let $\mathcal{X}$ be a compact metric space and let $C(\mathcal{X})$ denote the space of continuous functions on $\mathcal{X}$. Let $k(x, x')$ be the kernel defined as
\[
k(x, x') = \sigma^2 \exp\left(-\frac{\|\varphi(x) - \varphi(x')\|^2}{2\ell^2}\right),
\]
where the feature map $\varphi : \mathcal{X} \to \mathbb{R}^m$ is constructed from compactly supported, smooth bump functions centered at locations $\{x_i\}$ with tunable radii $\{r_i\}$ and amplitudes $\{a_i\}$\rev{, and assume that $\varphi$ is continuous and injective on $\mathcal{X}$ (guaranteed whenever the local distance profiles separate the points of $\mathcal{X}$; cf.\ Proposition~\ref{prop:injective})}. Then the associated reproducing kernel Hilbert space (RKHS) is dense in $C(\mathcal{X})$, i.e., for every $f \in C(\mathcal{X})$ and every $\epsilon > 0$, there exists a function $g$ in the RKHS such that
\[
\sup_{x \in \mathcal{X}} |f(x) - g(x)| < \epsilon.
\]
\end{theorem}

\begin{proof}
\rev{Since $\mathcal{X}$ is compact and $\varphi$ is continuous and injective,
$\varphi$ is a homeomorphism onto its image $Z = \varphi(\mathcal{X}) \subset
\mathbb{R}^m$, which is compact. The Gaussian kernel $k_{\mathrm{RBF}}(z, z') =
\exp\left(-\frac{\|z-z'\|^2}{2\ell^2}\right)$ is universal on compact subsets of
$\mathbb{R}^m$~\citep{micchelli2006universal, steinwart2001influence}, so its RKHS
$\mathcal{H}_{\mathrm{RBF}}$ is dense in $C(Z)$. Given $f \in C(\mathcal{X})$ and
$\epsilon > 0$, the function $f \circ \varphi^{-1}$ is continuous on $Z$; choose
$\tilde{g} \in \mathcal{H}_{\mathrm{RBF}}$ with $\sup_{z \in Z} |f(\varphi^{-1}(z))
- \tilde{g}(z)| < \epsilon$. The RKHS of the composed kernel $k(x, x') =
\sigma^2 k_{\mathrm{RBF}}(\varphi(x), \varphi(x'))$ consists exactly of functions
of the form $\tilde{h} \circ \varphi$ with $\tilde{h} \in
\mathcal{H}_{\mathrm{RBF}}$ (restricted to $Z$)~\citep[Ch.~4]{Scholkopf2002}, so
$g := \tilde{g} \circ \varphi$ lies in the RKHS of $k$ and
\[
\sup_{x \in \mathcal{X}} |f(x) - g(x)|
= \sup_{z \in Z} \left| f(\varphi^{-1}(z)) - \tilde{g}(z) \right| < \epsilon. \qedhere
\]}
\end{proof}

\subsection{Stability and Conditioning of the Kernel Matrix}

\begin{theorem}\label{thm:stability}
Let $\mathcal{X}$ be a metric space, and let $L = \{x_1, \ldots, x_m\} \subset \mathcal{X}$ be a set of $m$ landmarks, each with a compactly supported bump function embedding as in the previous theorem:
\[
\varphi(x) = \left[ b(d(x, x_1)),\, b(d(x, x_2)),\, \ldots,\, b(d(x, x_m)) \right]^\top \in \mathbb{R}^m,
\]
where $b(d)$ is nonzero if and only if $d < r$ for some fixed support radius $r$. Consider the kernel
\[
k(x, x') = \sigma^2 \exp\left(-\frac{\|\varphi(x) - \varphi(x')\|^2}{2\ell^2}\right)\,.
\]
Let $\{z_1, \ldots, z_N\}$ be a dataset, and let $K \in \mathbb{R}^{N \times N}$ be the Gram matrix with entries $K_{ij} = k(z_i, z_j)$. 

\rev{\textbf{Assumption (A1) (bounded overlap).}} Assume the expected number of overlapping nonzero entries in $\varphi(z_i)$ and $\varphi(z_j)$ is bounded above by $s \ll m$, independent of $m$ as $m$ increases.

Then, \rev{under Assumption (A1),} as $m$ grows, the Gram matrix $K$ remains well-conditioned: its condition number is bounded above by a constant depending on the maximal overlap $s$ and the kernel parameters, but not on $m$.
\end{theorem}

\begin{proof}\rev{[Structural sketch]}
As shown in the sparsity theorem, for each datapoint $z_i$, the embedding $\varphi(z_i)$ has at most $s$ nonzero entries. Further, for most pairs $(z_i, z_j)$, the supports of $\varphi(z_i)$ and $\varphi(z_j)$ do not overlap, so $\|\varphi(z_i) - \varphi(z_j)\|^2 = \|\varphi(z_i)\|^2 + \|\varphi(z_j)\|^2$.

Thus, most off-diagonal entries of $K$ take the form 
\[
k(z_i, z_j) = \sigma^2 \exp\left( -\frac{\|\varphi(z_i)\|^2 + \|\varphi(z_j)\|^2}{2\ell^2} \right) = k_0(z_i) k_0(z_j),
\]
where $k_0(z) = \exp\left( -\frac{\|\varphi(z)\|^2}{2\ell^2} \right)$. This structure yields a Gram matrix that is block-diagonal (or close to it) with small off-diagonal entries except within overlapping support, where blocks of size $s \times s$ may appear.

Block-diagonal or banded matrices with small block size always have condition numbers bounded by a constant (given by the maximal block condition number) and are thus resistant to the ill-conditioning that arises when all entries are dense and $m$ is large, as seen in standard high-dimensional RBF kernels.

Therefore, for any $m$, the condition number of $K$ is controlled by the overlap $s$ and the kernel parameters $(\sigma^2, \ell)$, but is not adversely affected by increasing $m$.
\end{proof}

\rev{\begin{remark}[Scope of Theorem~\ref{thm:stability}]
Assumption (A1) need not hold uniformly across datasets --- e.g., under strongly
clustered sampling with radii large relative to cluster diameters --- and the
argument above is structural rather than fully quantitative.
Appendix~\ref{app:conditioning} therefore verifies the predicted behavior
empirically, reporting Gram-matrix condition numbers as a function of training
size for the SLE embedding and the raw-distance embedding at their respective
likelihood-optimized hyperparameters.
\end{remark}}

\subsection{Sparsity Scaling of Embedding with Number of Landmarks}
\begin{theorem}\label{thm:sparsity_scaling}

Let $\mathcal{X}$ be a metric space endowed with distance $d(\cdot, \cdot)$, and let $L = \{x_1, \ldots, x_m\} \subset \mathcal{X}$ be a set of $m$ landmarks. For each $i$, let $r_i > 0$, and define the compactly supported bump function $b_i(x) = b(d(x, x_i); a_i, r_i, \beta_i)$ which is nonzero if and only if $d(x, x_i) < r_i$. For any $x \in \mathcal{X}$, define the embedding vector
\[
\varphi(x) = [b_1(x), \, b_2(x),\, \ldots,\, b_m(x)]^\top \in \mathbb{R}^m.
\]

Suppose $r_i = r$ for all $i$, and fix a probability measure $\mu$ on $\mathcal{X}$. Denote 
\[
p_m = \mathbb{P}_{x \sim \mu}\left[d(x, x_i) < r\right]
\]
(where by symmetry, this does not depend on $i$ if landmarks are spread in a regular fashion and $m$ is large).

Then the expected proportion of nonzero entries in $\varphi(x)$ for $x \sim \mu$ satisfies
\[
\mathbb{E}_{x \sim \mu} \left[\frac{\|\varphi(x)\|_0}{m}\right] = p_m,
\]
so the expected number of nonzero entries is $m p_m$. If the landmarks become dense but $r$ is fixed and small relative to the typical inter-point distance, then $p_m \ll 1$ and the embedding becomes increasingly sparse as $m$ grows.
\end{theorem}

\begin{proof}
For any $x \in \mathcal{X}$, the $i$-th entry of $\varphi(x)$ is nonzero if and only if $d(x, x_i) < r$. Thus,
\[
\|\varphi(x)\|_0 = \sum_{i=1}^m \mathbb{I}\{d(x, x_i) < r\}.
\]
Taking expectation over $x \sim \mu$, linearity of expectation gives
\[
\mathbb{E}_{x \sim \mu}[\|\varphi(x)\|_0] = \sum_{i=1}^m \mathbb{P}_{x \sim \mu}[d(x, x_i) < r].
\]
If the distribution of landmarks is regular and each $p_m := \mathbb{P}_{x \sim \mu}[d(x, x_i) < r]$ is (approximately) the same for all $i$, then
\[
\mathbb{E}_{x \sim \mu}[\|\varphi(x)\|_0] = m p_m.
\]
Dividing by $m$ yields the expected proportion. For small fixed $r$ compared to the domain size or typical landmark spacing, $p_m$ can be made arbitrarily small and does not increase with $m$. Thus, even as $m$ increases, the expected number of nonzero coordinates remains small compared to $m$, so the embedding is sparse.
\end{proof}

\subsection{Mitigation of Distance Concentration in High Dimensions}
\begin{theorem}\label{thm:concentration}
Let $\mathcal{X}$ be a space in which standard Euclidean embeddings are subject to distance concentration (i.e., as the feature dimension $m \to \infty$, pairwise distances between random points become nearly equal). Let $\varphi : \mathcal{X} \to \mathbb{R}^m$ be the compactly supported bump-function embedding defined as in previous theorems, so that each component $\varphi_i(x) = b(d(x, x_i); a_i, r_i, \beta_i)$ is nonzero if and only if $d(x, x_i) < r_i$ for landmark $x_i$.

Consider the kernel:
\[
k(x, x') = \sigma^2 \exp\left(-\frac{\|\varphi(x) - \varphi(x')\|^2}{2 \ell^2}\right).
\]
Then, as $m$ increases, provided the radii $\{r_i\}$ remain small relative to the domain, the following hold:
\begin{enumerate}
    \item \textbf{Locality.} The overlap $\langle \varphi(x), \varphi(x') \rangle$ is nonzero for a pair $(x, x')$ if and only if they fall within the support of at least one common bump, i.e., $d(x, x_i) < r_i$ and $d(x', x_i) < r_i$ for some $i$.
    \item \textbf{Suppression of Distance Concentration.} For most pairs $(x, x')$, $\varphi(x)$ and $\varphi(x')$ have disjoint support, so that $\|\varphi(x) - \varphi(x')\|^2 = \|\varphi(x)\|^2 + \|\varphi(x')\|^2$, making $k(x, x')$ small, often exactly zero. Only for nearby $x$, $x'$ will $k(x, x')$ be large.
    \item \textbf{Preservation of Informative Local Structure.} The nonzero entries in the kernel matrix reflect local neighborhoods determined by the supports of the bump functions, preserving meaningful similarity relations in high dimensions and overcoming the loss of discriminative power associated with distance concentration.
\end{enumerate}
Consequently, the kernel does not suffer from the distance concentration effect typically observed in high-dimensional Euclidean feature spaces. \rev{An empirical verification of the predicted conditioning behavior --- which would be the first casualty of distance concentration --- is provided in Appendix~\ref{app:conditioning}.}
\end{theorem}

\begin{proof}
1. By the construction of the bump embedding, $\varphi_i(x)$ is nonzero only if $d(x, x_i) < r_i$. Thus, for both $\varphi_i(x)$ and $\varphi_i(x')$ to be nonzero requires that both $x$ and $x'$ are within $r_i$ of $x_i$. If this is not the case for any $i$, then $\varphi(x)$ and $\varphi(x')$ have disjoint support.
\vspace{0.5em}

2. In high dimensions, for randomly selected $x, x'$, the likelihood that they share support in any coordinate $i$ (i.e., that $x$ and $x'$ both fall within the small ball of radius $r_i$ around $x_i$) is vanishingly small as $m$ increases, assuming the supports $r_i$ are fixed and small relative to the domain or inter-landmark distances. Thus, $\|\varphi(x) - \varphi(x')\|^2 = \|\varphi(x)\|^2 + \|\varphi(x')\|^2$ for most pairs, making $k(x, x')$ small or exactly zero except for local neighborhoods.
\vspace{0.5em}

3. Nontrivial (large) $k(x, x')$ values can only arise if there is substantial overlap in the supports of $\varphi(x)$ and $\varphi(x')$, i.e., $x$ and $x'$ are close to at least one common landmark. This means the kernel matrix is supported only on genuinely local neighborhoods, and entry magnitudes retain their informativeness even as $m$ grows. 

Therefore, the notorious phenomenon of distances becoming non-informative in high dimensions is avoided: the kernel remains locally discriminative and informative due to the sparsity and locality of the embedding.
\end{proof}

\subsection{Reducibility to Standard Stationary Kernels}
\begin{theorem}\label{thm:reduce}
Let $\mathcal{X}$ be an input space equipped with a distance function $d(\cdot, \cdot)$, and let $L = \{x_1, \ldots, x_m\} \subset \mathcal{X}$ be a set of landmarks. Define the embedding
\[
\varphi(x) = \left[ b(d(x, x_1); a_1, r_1, \beta_1),\; b(d(x, x_2); a_2, r_2, \beta_2),\; \ldots,\; b(d(x, x_m); a_m, r_m, \beta_m) \right]^\top,
\]
where each bump function $b(d; a, r, \beta)$ is continuous and strictly positive for $d < r$, and zero otherwise. Consider the kernel
\[
k(x, x') = \sigma^2 \exp\left( -\frac{ \|\varphi(x) - \varphi(x')\|^2 }{ 2\ell^2 } \right).
\]
Suppose for all $i$, $r_i \to \infty$ and $a_i, \beta_i$ are fixed so that $b(\cdot)$ becomes a globally supported, smooth, strictly positive function of $d(x, x_i)$. 

Then, for all $x, x' \in \mathcal{X}$,
\begin{enumerate}
    \item $\varphi(x)$ is a dense feature vector depending only on the set $\{d(x, x_i)\}_i$;
    \item $k(x, x')$ reduces to a function that depends on $\{d(x, x_i)\}_i$ and $\{d(x', x_i)\}_i$;
    \item If $d$ is a (conditionally) negative definite metric, then as $m \to \infty$ and with suitable choice of $b(\cdot)$, the kernel converges to a stationary RBF kernel $k_\text{RBF}(x, x') = \exp\left(-\frac{d(x, x')^2}{2 \tilde{\ell}^2}\right)$ on $(\mathcal{X}, d)$.
\end{enumerate}
\end{theorem}

\begin{proof}
1. When all $r_i \to \infty$, for any $x \in \mathcal{X}$ and any $i$, $d(x, x_i) < r_i$ always holds. Therefore, each coordinate $b(d(x, x_i); a_i, r_i, \beta_i)$ is strictly positive and only depends on $d(x, x_i)$.
\vspace{0.5em}

2. The vector $\varphi(x)$ encodes the global structure of $x$ with respect to all landmarks, and the difference $\varphi(x) - \varphi(x')$ depends only on the vector differences $\{ b(d(x, x_i)) - b(d(x', x_i)) \}_{i=1}^m$.

3. If $d$ is (conditionally) negative definite, the classic result for kernel methods states that the standard RBF kernel $k_\text{RBF}(x, x') = \exp\left(-\frac{d(x, x')^2}{2\tilde{\ell}^2}\right)$ is positive-definite and stationary on $(\mathcal{X}, d)$. For sufficiently large $m$ and appropriately chosen, smooth, global bump functions, the feature embedding $\varphi(x)$ can be made to approximate an injective mapping from $\mathcal{X}$ into $\mathbb{R}^m$ such that $\|\varphi(x) - \varphi(x')\|$ encodes $d(x, x')$ up to a scale. Thus, in the limit $r_i \to \infty$ and $m \to \infty$, $k(x, x')$ converges to the standard RBF kernel over $d(\cdot, \cdot)$.

Therefore, the bump-embedding kernel recovers the standard RBF kernel on the original space when the bumps become globally supported.
\end{proof}

\subsection{Continuity and Smoothness}
\begin{theorem}\label{thm:smooth}
Let $\mathcal{X}$ be a topological space, and let $d : \mathcal{X} \times \mathcal{X} \to \mathbb{R}$ be a continuous function. Consider a collection of smooth, compactly supported bump functions $b(d; a, r, \beta)$ that are $C^{\infty}$ (infinitely differentiable) on their support. Define the embedding
\[
\varphi(x) = \left[ b(d(x, x_1); a_1, r_1, \beta_1),\; b(d(x, x_2); a_2, r_2, \beta_2),\; \ldots,\; b(d(x, x_m); a_m, r_m, \beta_m) \right]^\top,
\]
for a set of fixed landmarks $\{ x_i \}_{i=1}^m$. The kernel is given by
\[
k(x, x') = \sigma^2 \exp\left( -\frac{ \| \varphi(x) - \varphi(x') \|^2 }{ 2\ell^2 } \right).
\]

If $d(x, x_i)$ is smooth in $x$, and $b$ is smooth in $d$, then $k(x, x')$ is smooth (infinitely differentiable) as a function of each argument.
\end{theorem}

\begin{proof}
Since $d(x, x_i)$ is assumed smooth in $x$ (for all fixed $x_i$), and $b(\cdot)$ is $C^\infty$ as a function of $d$, each coordinate of $\varphi(x)$ is a composition of smooth functions and hence is $C^\infty$ in $x$. Therefore, $\varphi(x)$ is $C^\infty$ as a mapping from $\mathcal{X}$ to $\mathbb{R}^m$.

The Euclidean norm, squaring, and difference are all smooth operations in $\mathbb{R}^m$, so $F(x, x') = \|\varphi(x) - \varphi(x')\|^2$ is a smooth function of both $x$ and $x'$. The function $k(x, x')$ is then a composition of $F(x, x')$ with the exponential function, which is also smooth.

Thus, $k(x, x')$ is smooth in both arguments; that is, $k \in C^{\infty}(\mathcal{X} \times \mathcal{X})$.
\end{proof}

\subsection{Empirical Scaling and Complexity}
\begin{theorem}\label{thm:scaling}
Let $\mathcal{X}$ be a metric space, let $L = \{x_1, \ldots, x_m\} \subset \mathcal{X}$ be $m$ landmarks, and let $b(d; a, r, \beta)$ denote a compactly supported bump function as in previous theorems. Define the embedding
\[
\varphi(x) = \left[ b(d(x, x_1); a_1, r_1, \beta_1),\, b(d(x, x_2); a_2, r_2, \beta_2),\, \ldots,\, b(d(x, x_m); a_m, r_m, \beta_m) \right]^\top \in \mathbb{R}^m.
\]
Let $s_x = \|\varphi(x)\|_0$ denote the number of nonzero entries in $\varphi(x)$. The kernel is given by
\[
k(x, x') = \sigma^2 \exp\left(-\frac{\|\varphi(x) - \varphi(x')\|^2}{2\ell^2}\right).
\]

Then:
\begin{enumerate}
    \item For any pair $x, x' \in \mathcal{X}$, computing $\|\varphi(x) - \varphi(x')\|^2$ and thus $k(x, x')$ requires $\mathcal{O}(s_{x,x'})$ operations, where $s_{x,x'}$ is the number of indices $i$ such that at least one of $\varphi_i(x)$ or $\varphi_i(x')$ is nonzero (i.e., at most $s_x + s_{x'}$).
    \item If all bump radii $r_i$ are small compared to the domain and the landmark set is sufficiently large, then $s_x \ll m$ for typical $x$, so computational cost is sublinear in $m$.
    \item The total number of nonzero entries in the $N \times m$ embedding matrix for $N$ data points is $\mathcal{O}(N \bar{s})$, where $\bar{s}$ is the average sparsity per embedding, and all kernel matrix and matrix operation costs (e.g., matrix-vector products) scale accordingly.
\end{enumerate}
Therefore, kernel evaluation and matrix operations scale with embedding sparsity (local bump overlap), not with the full ambient embedding dimension $m$. \rev{Note that this concerns the kernel and linear-algebra stage; the cost of forming the distance profiles themselves is discussed in Appendix~\ref{app:cost}.}
\end{theorem}

\begin{proof}
1. By construction, $b(\cdot)$ is compactly supported, so for each $x$ only a small fraction $s_x$ of the coordinates in $\varphi(x)$ are nonzero. The squared Euclidean distance $\|\varphi(x) - \varphi(x')\|^2$ involves only dimensions where at least one entry is nonzero (i.e., the union of nonzero indices in $\varphi(x)$ and $\varphi(x')$). Thus, its computation is $\mathcal{O}(s_{x,x'})$.

2. If bump radii are small and landmarks are widely dispersed, $s_x$ remains small and does not increase with $m$. Thus, the per-kernel evaluation and per-row storage cost are both $\mathcal{O}(s_x) \ll m$.

3. For a dataset of $N$ points, the total number of floating-point operations for forming all $\varphi(x^{(i)})$ is $\mathcal{O}(N \bar{s})$ for average sparsity $\bar{s}$. Matrix operations such as matrix-vector products with the Gram matrix $K$ also scale with the number of nonzero overlaps between pairs of embeddings, yielding $\mathcal{O}(N \bar{s})$ scaling for sparse kernels, far more efficient than the $\mathcal{O}(Nm)$ scaling of dense embeddings.

Thus, the complexity is governed by embedding sparsity rather than by the full embedding dimension $m$, as claimed.
\end{proof}

\section{The Non-Stationary SLE Kernel}
\label{app:ns_kernel}

In the \rev{stationary} formulation of the SLE kernel, each bump function $b(d(x, x_i); 
a_i, r_i, \beta_i)$ carries the same amplitude $a$, radius $r$, and shape 
parameter $\beta$. So far, we have treated these as free hyperparameters to be 
learned globally. However, a more powerful and principled choice is to let them 
vary as functions (arbitrary parametric, NNs, polynomial) of position in $\mathcal{X}$, making the kernel explicitly \emph{non-stationary}: the similarity structure it encodes can differ across 
different regions of the input space.

Concretely, we allow each landmark $x_i$ to carry its own local parameters
    $a_i = a(x_i), \quad r_i = r(x_i), \quad \beta_i = \beta(x_i),
    $
that are functions defined on 
$\mathcal{X}$. These functions can be specified by the user based on prior 
knowledge of the input domain, or learned from data. This yields the 
\emph{non-stationary SLE kernel}, in the case of RBF,
\begin{equation}
    k_{\mathrm{SLE-RBF}}^{\mathrm{NS}}(x, x') = \sigma^2 \exp\!\left( 
    -\frac{\|\phi^{\mathrm{NS}}(x) - \phi^{\mathrm{NS}}(x')\|^2}{2\ell^2} \right),
    \label{eq:sle_ns}
\end{equation}
where the non-stationary embedding is
\begin{equation}
    \phi^{\mathrm{NS}}(x) = \bigl[b(d(x,x_1);\,a(x_1),\,r(x_1),\,\beta(x_1)),\; 
    \ldots,\; b(d(x,x_{|\mathcal{D}|});\,a(x_{|\mathcal{D}|}),\,r(x_{|\mathcal{D}|}),\,\beta(x_{|\mathcal{D}|}))\bigr]^\top.
    \label{eq:ns_embedding}
\end{equation}

The non-stationarity enters entirely through the domain-varying \rev{bump parameters}, 
and the PSD property is unaffected, as the following proposition confirms.

\begin{proposition}[PSD of the Non-Stationary SLE Kernel]
\label{prop:ns_psd}
Let $a(\cdot)$, $r(\cdot)$, and $\beta(\cdot)$ be arbitrary positive-valued 
functions on $\mathcal{X}$. Then $k_{\mathrm{SLE}}^{\mathrm{NS}}$ as defined in 
Eq.~\eqref{eq:sle_ns} is a positive semi-definite kernel on $\mathcal{X}$, for 
any distance $d(\cdot,\cdot)$, whether or not it is CND.
\end{proposition}

\begin{proof}
The non-stationary embedding $\phi^{\mathrm{NS}} : \mathcal{X} \rightarrow 
\mathbb{R}^{|\mathcal{D}|}$ is a mapping into Euclidean space, regardless of how its 
parameters vary across $\mathcal{X}$. By the argument of 
Theorem~\ref{thm:psd}, any kernel of the form $k(x,x') = \rev{h}(\phi(x), 
\phi(x'))$ with \rev{$h$} PSD on $\mathbb{R}^{|\mathcal{D}|}$ is PSD on $\mathcal{X}$. Since 
the RBF kernel is PSD on $\mathbb{R}^{|\mathcal{D}|}$, the result follows immediately.
\end{proof}

The three parameter fields $a(\cdot)$, $r(\cdot)$, and $\beta(\cdot)$ each control a distinct aspect of the non-stationarity and can be defined via any parametric function to avoid an excessive number of hyperparameters. \rev{The individual roles of the parameter fields are described next.}

\subsection{Non-Stationary Parameter Fields} \label{app:ns}
\textbf{Radius $r(x_i)$.} The support radius controls the \emph{spatial reach} of 
landmark $x_i$: how large a neighborhood around $x_i$ contributes to the 
similarity structure. In regions where the function being modeled varies rapidly, 
smaller radii are appropriate, encoding the intuition that only very nearby points 
should be considered similar. In smoother regions, larger radii allow information 
to propagate further. Varying $r(\cdot)$ therefore adapts the effective length 
scale of the kernel to local function complexity, analogously to the 
input-dependent length scales of non-stationary kernels such as those proposed by 
\citet{Paciorek2003}.

\textbf{Amplitude $a(x_i)$.} The amplitude controls the \emph{contribution} of 
landmark $x_i$ to the overall embedding. Landmarks in regions of high data density 
or high functional relevance can be upweighted, while those in sparse or 
uninformative regions can be downweighted. This provides a mechanism for the kernel 
to allocate representational capacity unevenly across $\mathcal{X}$, analogously 
to signal variance modulation in non-stationary GP models.

\textbf{Shape $\beta(x_i)$.} The shape parameter controls the \emph{profile} of 
the bump: how steeply similarity decays with distance from $x_i$ within the 
support. Large $\beta$ produces a bump that is nearly flat near $x_i$ and drops 
sharply near the boundary $r$, while small $\beta$ produces a smoother, more 
gradual decay. Varying $\beta(\cdot)$ therefore allows the kernel to encode 
different local smoothness assumptions in different parts of $\mathcal{X}$.

\textbf{Signal Variance $\sigma(x)$.} In addition, we can make the signal variance non-stationary
by considering $\sigma^2=\sigma^2(x)=\sigma(x) \sigma(x)$\rev{.}

Together, these four spatially varying parameters give the non-stationary SLE 
kernel considerable flexibility. We note that the stationary SLE kernel is recovered as the special case where $a(\cdot)$, 
$r(\cdot)$, and $\beta(\cdot)$ are constant functions.

\begin{remark}[Sparsity Under Non-Stationarity]
The sparsity properties established in Theorems~\ref{thm:sparsity} 
and~\ref{thm:sparsity_scaling} carry over directly to the non-stationary case. 
For any $x \in \mathcal{X}$, the $i$-th entry of $\phi^{\mathrm{NS}}(x)$ is 
nonzero if and only if $d(x, x_i) < r(x_i)$. The expected number of nonzero 
entries is therefore $\sum_{i=1}^{|\mathcal{D}|} \mathbb{P}_{x \sim \mu}[d(x, x_i) < 
r(x_i)]$, which remains small provided the radii $r(x_i)$ are small relative to 
the typical inter-point spacing. Non-stationarity in $r(\cdot)$ thus affects the 
\emph{local} sparsity pattern but does not compromise the overall sparsity of the 
embedding.
\end{remark}

\section{Complete Manifold Benchmarking Results}\label{app:manifold_results}

Here we present the full performance curves across all evaluated training dataset 
sizes for the GP regression experiments on the Dragon and Teddy Bear manifolds introduced 
in Section~\ref{sec:experiments}. These figures complement the summary statistics reported 
in Tables~\ref{tab:dragon_comparison} and~\ref{tab:teddy_comparison} of the main text, 
and provide a complete view of how predictive accuracy and uncertainty quantification 
evolve with training dataset size.

For the Dragon manifold, the Riemannian Mat\'{e}rn kernel approximates the covariance matrix via a
truncated spectral expansion $K = \Phi_X \operatorname{diag}(S) \Phi_X^\top$,
where $l$ eigenpairs are retained~\citep{borovitskiy2020Matern}. This
truncation causes the kernel matrix to become rank-deficient when the number of
training points $n$ approaches $l$, leading to numerical failure of the
Cholesky decomposition and, consequently, of model training. The resulting
instability is visible as sharp spikes in RMSE and CRPS at $n \approx l$
($n = \{100, 500, 1000\}$) in Figure~\ref{fig:dragon_comparison}b--c, where
divergent values are indicated by dashed lines rather than reported numerically.
Most critically, this instability corrupts uncertainty quantification: the
posterior variance becomes negative, requiring it to be clamped to zero and
collapsing the predictive distribution to a point estimate. This renders CRPS
undefined and drives PICP to 0\%, explaining the empty cells in
Table~\ref{tab:dragon_comparison}. In contrast, the SLE kernel maintains a
stable PICP near the nominal 95\% level and a smoothly decreasing CRPS with
increasing training size, demonstrating reliable, calibrated uncertainty
quantification without a spectral truncation parameter. For the Teddy Bear manifold, the full results are shown in Figure \ref{fig:teddy_comparison}. In contrast to the previous tests, the Geometric kernel is stable across dataset sizes, so these results represent robust, stable runs.  

\begin{figure}[htbp]
    \centering
    \includegraphics[width=\textwidth]{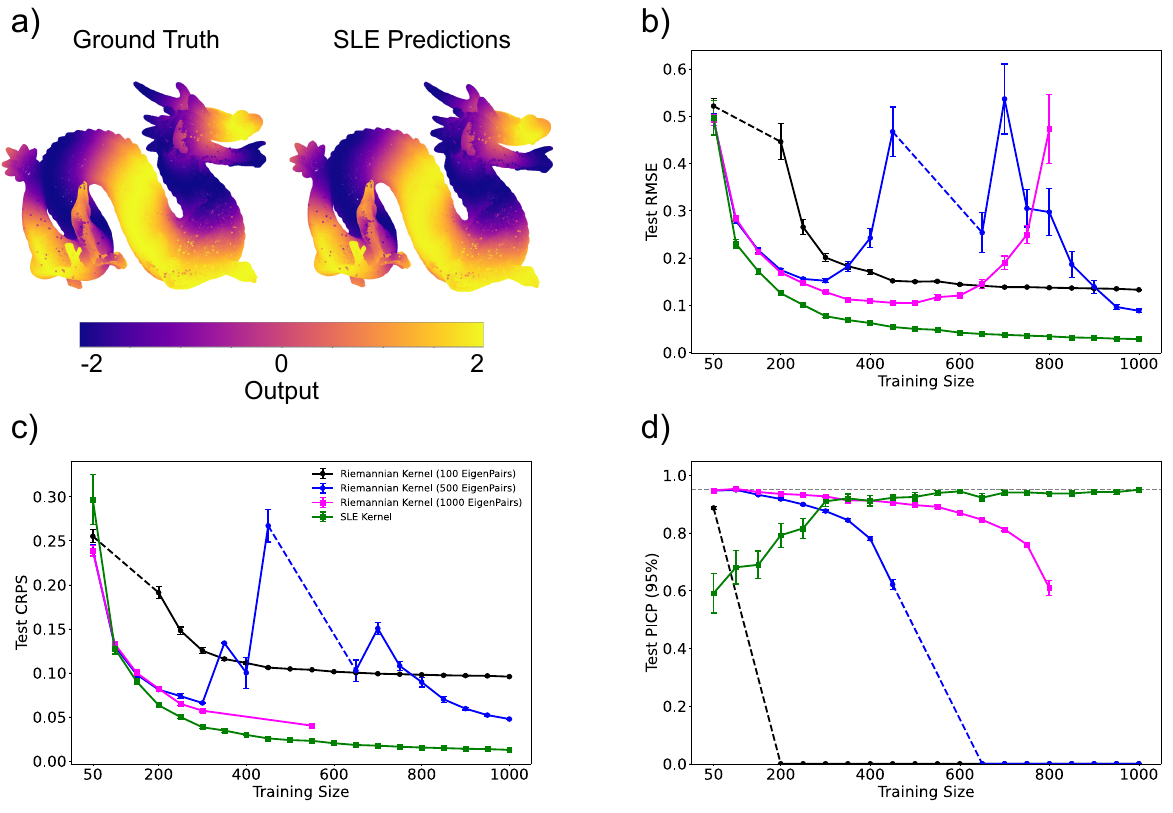}
    \caption{Benchmarking the SLE kernel against the Riemannian Mat\'{e}rn 
    kernel~\citep{borovitskiy2020Matern} on the Dragon manifold. (a) Ground truth and 
    SLE-predicted output values represented by the color of the mesh vertices. 
    (b) Test RMSE, (c) CRPS, and (d) PICP (95\% interval) as a function of training 
    dataset size for the SLE kernel and the Riemannian Mat\'{e}rn kernel with 100, 500, 
    and 1000 eigenpairs. Dashed lines indicate training sizes where results are omitted 
    due to numerical divergence caused by ill-conditioning of the Riemannian kernel 
    matrix when $n \approx l$; this instability also accounts for the degraded 
    performance of the 500-eigenpair variant relative to the 100-eigenpair variant 
    at certain training sizes. Error bars represent the standard error of the mean 
    across 30 trials.}
    \label{fig:dragon_comparison}
\end{figure}

\begin{figure}[htbp]
    \centering
    \includegraphics[width=\textwidth]{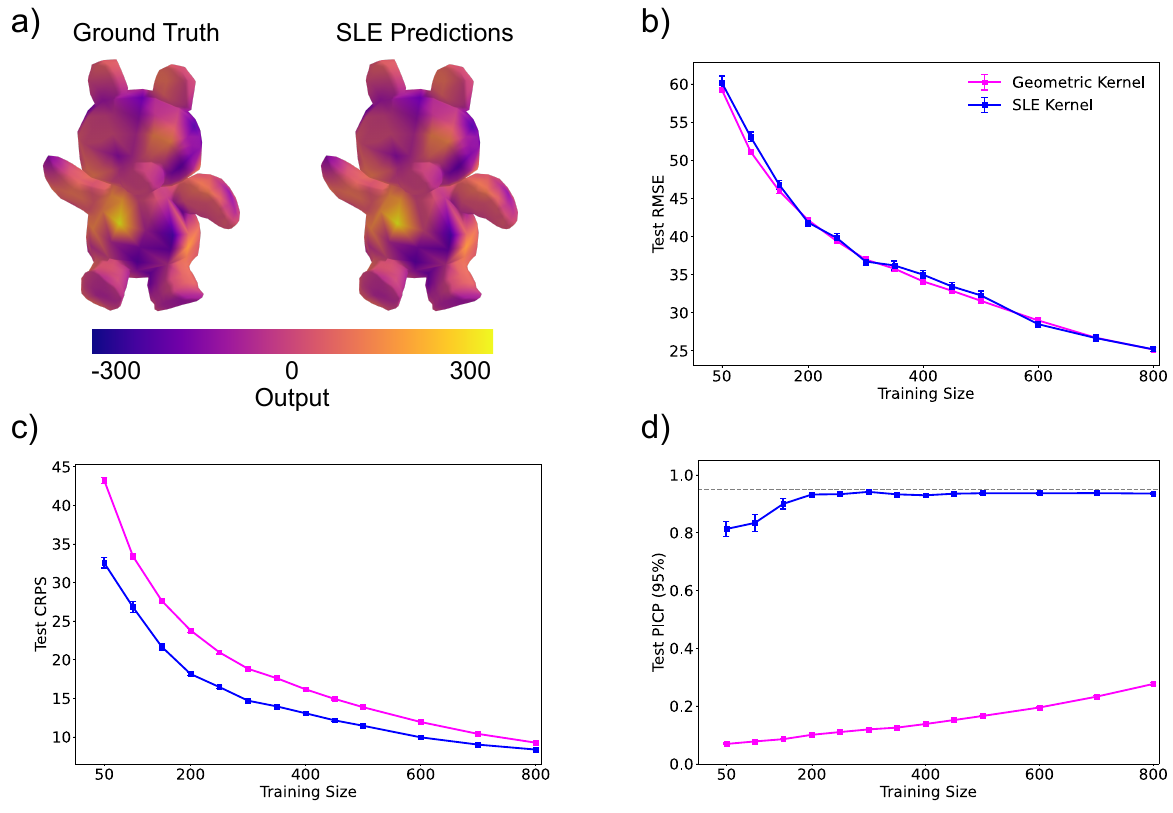}
    \caption{Benchmarking the SLE kernel against the Geometric 
    kernel~\citep{Mostowsky2025} on the Teddy Bear manifold. (a) Ground truth and SLE 
    predicted output values represented by the color of the mesh vertices. 
    (b) Test RMSE, (c) CRPS, and (d) PICP (95\% interval) as a function of training 
    dataset size for the SLE kernel and the Geometric kernel. Error bars represent the 
    standard error of the mean across 30 trials.}
    \label{fig:teddy_comparison}
\end{figure}

\section{Additional Information on Experiments}

\subsection{Dragon Manifold}
\label{app:more_info_dragon}

\textbf{Dataset}. The experiment is conducted on the Stanford Dragon, a standard 3D benchmark geometry represented as a triangulated surface mesh. The mesh and associated scalar field values are generated using Firedrake. The mesh contains two arrays: `vertices' of shape $(N, 3)$ storing the 3D Cartesian coordinates of all $N = 100{,}179$ mesh vertices, and `ground truth' of shape $(N,)$ storing the corresponding output. A precomputed symmetric geodesic distance matrix of shape $N \times N$ is also required. Each entry $(i,j)$ contains the shortest-path distance between vertex $i$ and vertex $j$ measured along the mesh surface. The exact geodesic distance matrix was computed once using a graph-based shortest-path algorithm and stored for later extraction. This pre-computation took approximately 12 hours on a single CPU. The dominant memory cost is the $N \times N$ pairwise distance matrix and the derived covariance matrix, both of which store $N^2$  floating-point values. Once the distance matrix is available, training the GP model is inexpensive: with 100 training points, hyperparameter optimization completes in under a minute, scaling to roughly 3–4 minutes for 1000 points on a single CPU.

\textbf{Experimental Design}. To assess how performance scales with training set size, the GP is trained across multiple sizes, with 30 independent random trials per size. Training configurations (random seeds and training indices) are pre-generated and stored in a JSON file. A fixed held-out test set of 5,000 points is shared across all experiments to ensure consistent evaluation. An initial single exploratory run (seed 2355, $n_{\text{train}} = 1{,}000$, $n_{\text{test}} = 5{,}000$) is first conducted using global hyperparameter optimization to verify the setup before launching the full experiment.

\textbf{GP Model Setup}. The GP model is implemented using the `gpCAM` library with three components. They are implemented to match the Riemannian Kernel implementation discussed in~\citep{borovitskiy2020Matern}.

\textit{Prior Mean}. A zero prior mean is used throughout.

\textit{Noise}. A homoscedastic zero-mean white noise with variance of $10^{-5}$ is used, reflecting non-noisy data, while still ensuring numerical stability.

\textit{Kernel}. The geodesic SLE kernel is used. Each input point is mapped to a feature vector by evaluating the bump function at its geodesic distances to all $n_{\text{train}}$ training points, and a $\nu=3/2$ Matérn kernel is then applied to the Euclidean distance between these embeddings. Concretely, for a point $\mathbf{x}$, the $i$-th component of its feature vector is:

$$\phi(x)_i = \text{bump}(d_g(x, z_i),\ r,\ \beta=1), \quad i = 1, \ldots, n_{\text{train}}$$

where $d_g$ denotes geodesic distance and the bump function is:

$$\text{bump}(d, r, \beta) = \begin{cases} \exp\!\left(\dfrac{-\beta}{1 - d^2/r^2} + \beta\right) & \text{if } d < r \\ 0 & \text{otherwise} \end{cases}$$

The kernel value between two points is then:

$$k(x, x') = \sigma_f^2 \cdot \left(1 + \frac{\sqrt{3}\,\|\phi(x) - \phi(x')\|_2}{\ell}\right) \exp\!\left(-\frac{\sqrt{3}\,\|\phi(x) - \phi(x')\|_2}{\ell}\right)$$

The kernel has three hyperparameters: signal variance $\sigma_f^2$, bump radius $r$, and Matérn length scale $\ell$. Since `gpCAM` passes 3D coordinates rather than vertex indices to the kernel, vertex indices are recovered at evaluation time via a k-d tree nearest-neighbor lookup over the full vertex array.

\textbf{Hyperparameter Bounds and Initialization}. Bounds are set adaptively per trial. The signal variance is bounded between $0.01 \cdot \text{Var}(y_{\text{train}})$ and $10 \cdot \text{Var}(y_{\text{train}})$, anchoring it to the observed scale of the target field. The bump radius is bounded between the minimum and maximum non-zero pairwise geodesic distances within the training set. The Matérn length scale is given broad, uninformative bounds of $[0.01,\ 100]$. All hyperparameters are initialized to the midpoint of their respective bounds.

\textbf{Training and Evaluation}. Hyperparameters are optimized by maximizing the log marginal likelihood. Optimization is performed using Markov Chain Monte Carlo (MCMC) with up to 4,000 iterations, providing more thorough exploration of the hyperparameter space. Results are saved incrementally after each trial, so the experiment can be interrupted and resumed without data loss.

Three metrics are computed on the held-out test set: \textit{RMSE} on the test set to measure fit and generalization; \textit{CRPS} (mean and standard deviation) as a proper scoring rule evaluating the full predictive distribution; and \textit{PICP} at the 95\% level, measuring the fraction of test points covered by the posterior predictive interval. A well-calibrated model should achieve PICP $\approx 0.95$.

\subsection{Teddy Bear Manifold}
\label{app:additional_info_on_teddy}
\textbf{Dataset}. The second experiment is conducted on a teddy bear mesh, loaded from an .obj file using the Mesh class from the Geometric Kernels library \citep{Mostowsky2025}. The geodesic distance matrix is computed using the same graph-based shortest-path approach as above and loaded directly into memory. This pre-computation took approximately one \rev{hour} on a single CPU. Once the distance matrix is available, training the GP model on the teddy bear dataset takes approximately 4 seconds for 100 training points and 2921 seconds (~49 minutes) for 1000 training points on a single CPU.

\textbf{Experimental Design}. The design mirrors that of the Dragon experiment: multiple training set sizes are evaluated with 30 independent random trials per size, using a fixed held-out test set across all trials. An initial exploratory run is performed with seed 3256, $n_{\text{train}} = 50$, and $n_{\text{test}} = 50$, using global optimization to verify the setup.

\textbf{GP Model Setup}. The same GP framework is used, with the following differences.

\textit{Input representation}. Rather than passing 3D vertex coordinates to the kernel, vertex indices are passed directly as integer-valued inputs of shape $(n, 1)$. This eliminates the need for the k-d tree nearest-neighbor lookup used in the Dragon experiment, since geodesic distances can be indexed directly.

\textit{Bump amplitude}. Unlike the Dragon experiment, where the bump amplitude was fixed at 1, here it is treated as a free hyperparameter $a$, allowing the model to control the scale of the feature embedding independently of the signal variance.

\textit{Prior Mean}. A zero prior mean is used throughout.

\textit{Noise}. Rather than a fixed noise level, the noise variance is also treated as a learnable hyperparameter, reflecting greater uncertainty about the noise level in this dataset.

\textit{Kernel}. The same geodesic SLE kernel structure is used, with the bump function defined as above. However, the inner stationary kernel is replaced with a $\nu=5/2$ Matérn:
$k(x, x') = \sigma_f^2 \cdot \left(1 + \frac{\sqrt{5}\,D}{\ell} + \frac{5D^2}{3\ell^2}\right)\exp\!\left(-\frac{\sqrt{5}\,D}{\ell}\right), \quad D = \|\phi(x) - \phi(x')\|_2$.
The $\nu=5/2$ Matérn is used here to be consistent with the kernel order adopted in \citep{Mostowsky2025}.

\textbf{Hyperparameter Bounds}. The model has five hyperparameters: signal variance $\sigma_f^2 \in [10^2, 10^6]$; bump radius $r$ bounded by the minimum and maximum non-zero pairwise geodesic distances within the training set; bump amplitude $a \in [0.1, 50]$; Matérn length scale $\ell \in [10^{-3}, 40]$; and noise variance $\in [10^{-6}, 10]$. All hyperparameters are initialized at the midpoint of their bounds.

\textbf{Training and Evaluation}. Hyperparameters are optimized by maximizing the log marginal likelihood using global optimization with up to 4,000 iterations. The same three metrics are reported: test RMSE, CRPS, and PICP at the 95\% level.

\subsection{X-Ray Scattering Data Disguised as Distributions}
\rev{\label{app:additional_info_saxs}}

\textbf{Dataset}. The third experiment uses a dataset of 500 synthetic Small-Angle X-ray Scattering (SAXS) images designed to mimic real-world patterns from oriented soft-matter thin films measured at synchrotron beamlines; examples are shown in Appendix Figure~\ref{fig:saxs}. The target output is the effective elastic modulus, $y$, computed along the measurement axis. Two precomputed pairwise distance matrices between images are used: a Wasserstein distance (WD) matrix and a Sliced Wasserstein distance (SWD) matrix with 50 projections, both of full size $n \times n$ where $n$ is the total number of images. The experiment is run separately for each metric, and results are compared against a baseline GP using a standard  $\nu=3/2$ Matérn kernel applied directly to the sliced Wasserstein distances, without the bump embedding. The Wasserstein and sliced Wasserstein distance matrices were precomputed once and stored for later use; this pre-computation took approximately one hour on 32 CPUs. Once the distance matrix is available, training the GP model on the SAXS dataset takes approximately 11 seconds for 50 training points and 537 seconds (~9 minutes) for 500 training points on a single CPU.

\begin{figure}[ht]
    \centering
    \includegraphics[width=\linewidth]{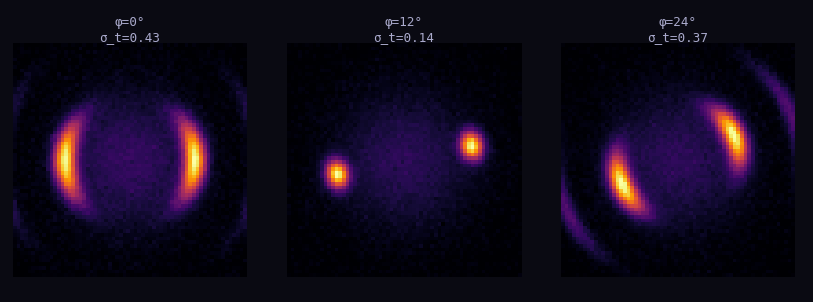}
    \caption{Three examples of the 500 SAXS images for our computational experiments.}
    \label{fig:saxs}
\end{figure}

\textbf{Experimental Design}. The same multi-trial design is used, with 30 independent random trials per training set size and a fixed held-out test set across all trials. An initial exploratory run is performed with a 70/30 train/test split (random state 3345) and global optimization to verify the setup. As in the mesh experiments, image indices are passed as integer-valued inputs of shape $(n, 1)$, and the selected distance matrix is indexed directly.

\textbf{GP Model Setup}. The SLE kernel with the bump function and $\nu=3/2$ Matérn inner kernel is used, as described above. The key structural difference from the mesh experiments is that the pairwise distances here are not geodesic distances on a surface but rather optimal transport distances between image distributions, specifically WD or SWD. The bump embedding therefore maps each image into a feature vector encoding its transport-distance neighborhood structure relative to the training set.
Two additional hyperparameters are introduced compared to the Dragon experiment. The prior mean is no longer fixed at zero but is instead a learnable constant $\mu_0$, bounded between the minimum and maximum observed training output values. This is appropriate here since the elastic modulus has a non-zero global mean that may vary across trials. The noise variance is also treated as a free hyperparameter, as in the teddy bear experiment.

\textbf{Hyperparameter Bounds}. The model has six hyperparameters in total: signal variance $\sigma_f^2 \in [1, 10^{10}]$; bump radius $r$ bounded between the minimum and twice the maximum non-zero pairwise distance within the training set (the upper bound is doubled to allow the bump to cover the full range of distances); bump amplitude $a \in [0.01, 10]$; Matérn length scale $\ell \in [10^{-4}, 100]$; noise variance $\in [10^{-5}, 10]$; and mean offset $\mu_0 \in [0, 10]$. All hyperparameters are initialized at the midpoint of their bounds.

\textbf{Training and Evaluation}. Hyperparameters are optimized by maximizing the log marginal likelihood using global optimization with up to 10,000 iterations. The same three metrics are reported: test RMSE, CRPS, and PICP at the 95\% level.

\section{Ablation Study}
\label{app:ablation}

To assess the contribution of the bump function in the SLE kernel, we conduct an ablation study comparing two embedding strategies: the proposed bump embedding, which applies a compactly supported smooth mask to the geodesic distances, and a distance-based embedding, which uses the raw geodesic distance vector to train landmarks directly. \rev{We note that the distance-based embedding is precisely the raw-distance embedding map underlying D2KE~\citep{wu2018d2ke} and classical landmark constructions, so this ablation doubles as a controlled empirical comparison against that family, isolating the contribution of the bump construction (cf.\ Section~\ref{sec:relatedwork}).} All other components of the kernel are held identical, including the inner Matérn covariance and the hyperparameter optimization procedure. We evaluate both variants on the Dragon and Teddy Bear meshes across a range of training set sizes, using three metrics: RMSE for predictive accuracy, CRPS for probabilistic sharpness, and PICP at the 95\% level for uncertainty calibration. The two manifolds offer complementary perspectives: the Dragon presents a geometrically intricate surface with thin features, while the Teddy Bear is a smoother, more compact shape.

\subsection{Dragon Manifold}

For the Dragon ablation, we use a target function constructed as a sum of sinusoids of geodesic distances from multiple well-separated source vertices with incommensurate periods. This deconfounded ground truth ensures that the function is genuinely manifold-defined but cannot be reduced to a one-dimensional function of distance from any single source, providing a fair test of both embeddings.

Figure~\ref{fig:ablation_dragon_comparison} reports the ablation results on the Dragon mesh. In terms of point prediction Figure~\ref{fig:ablation_dragon_comparison}(a), the two embeddings track each other closely across the entire training range, with no meaningful difference in RMSE. The same pattern is observed in CRPS Figure~\ref{fig:ablation_dragon_comparison}(b), where the two methods produce nearly identical probabilistic sharpness throughout. The picture changes for uncertainty calibration Figure~\ref{fig:ablation_dragon_comparison}(c). The bump embedding reaches the nominal 95\% PICP earlier and more reliably than the distance-based embedding, with a clear advantage maintained up to approximately 650 training points. Beyond that point, both methods converge toward the target coverage, and the distinction becomes minor. This indicates that, even when the two methods produce comparable point predictions, the sparse bump representation yields better-calibrated predictive variances across most of the practically relevant training range — a regime where principled uncertainty quantification is most valuable.

\begin{figure}[htbp]
\centering
\includegraphics[width=\textwidth]{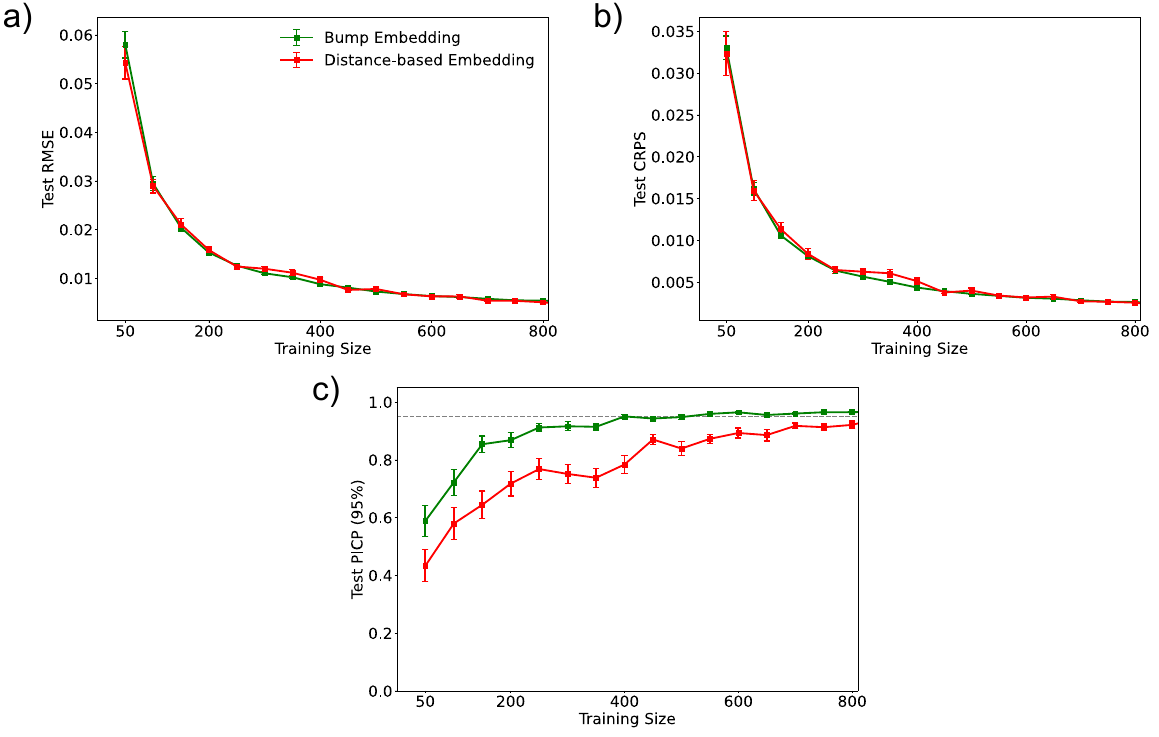}
\caption{Ablation study comparing bump embedding and distance-based embedding on the Dragon mesh. (a) Test RMSE, (b) Test CRPS, and (c) Test PICP at the 95\% confidence level, each as a function of training set size, averaged over 30 independent trials with error bars denoting one standard error. The two embeddings achieve essentially identical point prediction and probabilistic sharpness, while the bump embedding produces better-calibrated predictive intervals up to approximately 650 training points, after which the two methods converge to the nominal 95\% coverage.}
\label{fig:ablation_dragon_comparison}
\end{figure}

\subsection{Teddy Bear Manifold}

For the Teddy Bear ablation, we use the same ground truth as in the main experiments (see Appendix~\ref{app:additional_info_on_teddy} for full experimental details), allowing the ablation to be interpreted directly in the context of the corresponding evaluation. Figure~\ref{fig:ablation_teddy_comparison} reports the results. In terms of RMSE Figure~\ref{fig:ablation_teddy_comparison}(a), the distance-based embedding is competitive at small training sizes, where the embedding dimensionality remains manageable relative to the number of observations. As the training set grows, however, the raw distance embedding operates in an increasingly high-dimensional feature space without any regularization of its structure, and predictive accuracy degrades relative to the bump embedding. The bump embedding acts as a sparse, locally adaptive dimensionality reduction: each point is described only by its relationships to nearby landmarks, producing a compact, geometrically meaningful representation that remains well-conditioned as the training size scales. The bump embedding opens a growing advantage beyond approximately 400 training points and achieves substantially lower error at 800 points.

The benefit of the bump embedding extends to uncertainty quantification on this manifold. Figure~\ref{fig:ablation_teddy_comparison}(b) shows that while the distance-based embedding achieves marginally lower CRPS at very small training sizes, the bump embedding overtakes it at approximately 200 training points and maintains consistently superior probabilistic sharpness thereafter, with the gap widening at larger training sizes. Figure~\ref{fig:ablation_teddy_comparison}(c) further confirms this picture: both methods converge toward the nominal 95\% PICP, but the bump embedding does so faster, with tighter error bars and more stable calibration across the full training size range.

Taken together, the Dragon and Teddy Bear ablations show that the relative merit of the bump embedding depends on the geometric complexity of the manifold and the training regime. On smoother manifolds such as the Teddy Bear, the bump embedding yields clear improvements in both prediction and uncertainty quantification at moderate to large training sizes. On more intricate manifolds such as the Dragon, the bump embedding matches the distance-based embedding in point prediction and probabilistic sharpness while providing better-calibrated uncertainty estimates across most of the training range. In both cases, the bump embedding provides better-calibrated uncertainty in the small-data regime, where principled uncertainty quantification matters most.

\begin{figure}[htbp]
    \centering
    \includegraphics[width=\textwidth]{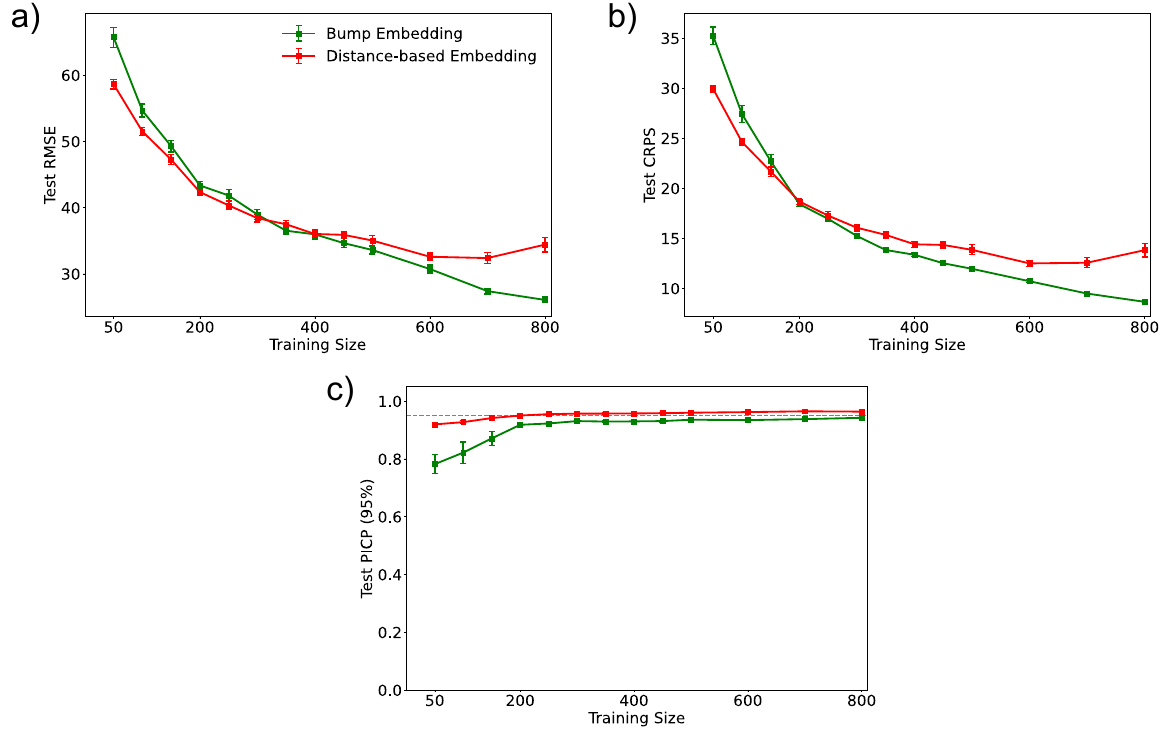}
    \caption{Ablation study comparing bump embedding and distance-based embedding on the teddy bear mesh. (a) Test RMSE, (b) Test CRPS, and (c) Test PICP at the 95\% confidence level, each reported as a function of training set size (50–800 points), averaged over 30 independent trials with error bars denoting one standard error. The distance-based embedding is competitive at small training sizes but degrades relative to the bump embedding as the feature space dimensionality grows with the number of landmarks. The bump embedding, which induces a sparse and locally adaptive representation of the manifold geometry, achieves lower prediction error and better-calibrated uncertainty estimates at moderate to large training sizes, with the crossover occurring at approximately 400 points for RMSE and 200 points for CRPS.}
    \label{fig:ablation_teddy_comparison}
\end{figure}

\rev{\section{Sensitivity Analyses}\label{app:sensitivity}

This section addresses the sensitivity of the SLE kernel to its central design choices: the bump support radius $r$, the bump amplitude $a$, and the choice of inner kernel applied to the embedding. We recall that in all main experiments the radius and amplitude are not hand-tuned but learned by marginal-likelihood maximization with data-adaptive bounds (Appendix~\ref{app:additional_info_on_teddy}), while the inner kernel family (e.g., Mat\'ern $\nu$ = 3/2 or $\nu$ = 5/2) is fixed by design choice, matched to the kernel order of the corresponding baseline; the analyses here characterize how performance varies away from the likelihood-selected radius and amplitude values, and how sensitive results are to the inner kernel choice itself. All three sweeps are conducted on the Teddy Bear manifold, following the same optimization procedure described in Appendix~\ref{app:additional_info_on_teddy}; the specific training size and number of independent trials used for each sweep are stated in the corresponding subsection below. A fixed held-out test set is shared across all grid points and trials in every sweep, consistent with the evaluation protocol used elsewhere in the paper.

\subsection{Bump Radius $r$}
The radius sweep uses a fixed training size of 300 points, the $\nu = 5/2$ Mat\'ern inner kernel (consistent with the main Teddy Bear experiment, Appendix~\ref{app:additional_info_on_teddy}), and 15 independent random trials per grid point. At each grid point, the radius $r$ is held fixed at its grid value---swept between the minimum and maximum nonzero pairwise geodesic distances within the training set (the same data-adaptive bounds used for $r$ during marginal-likelihood optimization elsewhere in the paper), up to $r \approx 56$---while the remaining hyperparameters (signal variance, bump amplitude, Mat\'ern length scale, and noise variance) are re-optimized by marginal-likelihood maximization, following the same optimization procedure described in Appendix~\ref{app:additional_info_on_teddy}. This isolates the effect of the radius after the model has been allowed to compensate through its remaining degrees of freedom, rather than showing raw sensitivity under an otherwise frozen model.

Figure~\ref{fig:radius_ablation_teddy_comparison} reports test RMSE, CRPS, and PICP as a function of $r$, together with the corresponding re-optimized log marginal likelihood. Predictive accuracy is highly sensitive to $r$ at the small end of the range: as $r$ shrinks toward zero, each embedding coordinate activates only a vanishingly small neighborhood, starving the kernel of local distance information even after re-optimizing the remaining hyperparameters, and both RMSE and CRPS rise sharply. Both metrics reach a minimum around $r \approx 13$–$15$ and then increase only mildly and monotonically thereafter, settling into a broad, shallow plateau for $r$ beyond approximately 20 that persists to the upper bound of the sweep. PICP shows a markedly different pattern: coverage is reasonable near $r = 0$, dips sharply to roughly 0.50 at very small nonzero radii---a regime in which the embedding, even with the remaining hyperparameters re-optimized, is expressive enough to fit the mean well but too locally constrained to produce well-calibrated variance estimates---and then recovers quickly, exceeding 0.90 by $r \approx 10$ and drifting upward toward the nominal 0.95 level as $r$ grows further, with the closest approach to the target near the upper end of the sweep.

\begin{figure}[htbp]
    \centering
    \includegraphics[width=\textwidth]{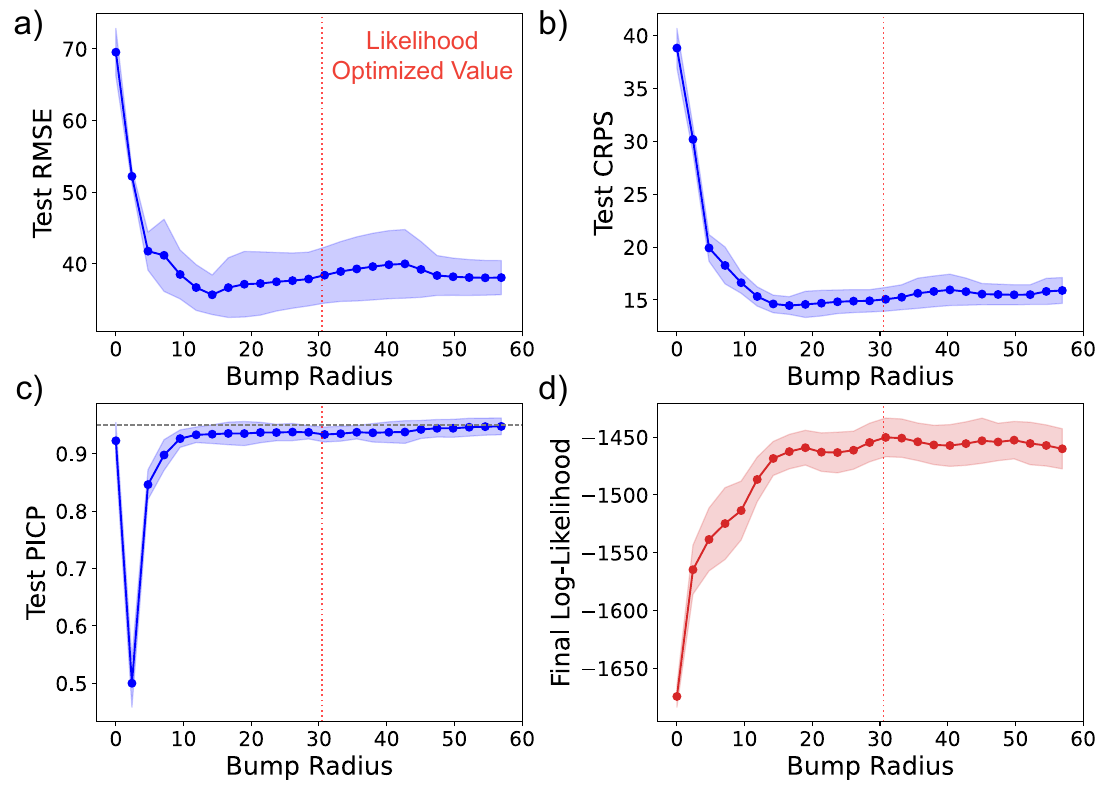}
    \caption{Sensitivity of the SLE kernel to the bump radius $r$ on the Teddy Bear manifold, at a fixed training size of 300 points with the $\nu = 5/2$ Mat\'ern inner kernel and 15 independent trials per grid point. At each grid point $r$ is held fixed while the remaining hyperparameters (signal variance, bump amplitude, Mat\'ern length scale, and noise variance) are re-optimized by marginal-likelihood maximization, isolating the effect of $r$ after the model has been allowed to compensate through its other degrees of freedom. (a) Test RMSE, (b) Test CRPS, (c) Test PICP at the 95\% level (dashed horizontal line marks nominal coverage), and (d) the corresponding re-optimized log marginal likelihood, each as a function of $r$. The red dotted vertical line marks the radius selected by unconstrained marginal-likelihood maximization in the main Teddy Bear experiment (Appendix~\ref{app:additional_info_on_teddy}). Shaded bands denote one standard error across trials.}
    \label{fig:radius_ablation_teddy_comparison}
\end{figure}

The re-optimized log-likelihood tracks this same transition (Figure~\ref{fig:radius_ablation_teddy_comparison}d): it rises sharply from its worst value at $r \approx 0$ and plateaus for $r$ beyond roughly 15–20, mirroring the RMSE/CRPS plateau and indicating that the marginal-likelihood surface itself, not just the point-prediction metrics, favors radii in this broad mid-to-large range over very small ones. Taken together, these results indicate a mild tension between sharpness and calibration: the radius minimizing RMSE and CRPS ($r \approx 13$–$15$) is somewhat smaller than the radius optimizing PICP, though the accuracy cost of choosing a larger, better-calibrated radius is small, since RMSE and CRPS remain within the flat plateau across this region. The value selected by unconstrained marginal-likelihood maximization in the main Teddy Bear experiment ($r \approx 30$, dotted line in Figure~\ref{fig:radius_ablation_teddy_comparison}) falls within this plateau, consistent with the intended role of $r$ as a data-adaptive length-scale analog rather than a parameter requiring manual tuning.

\subsection{Bump Amplitude $a$}

The amplitude sweep uses the same protocol as the radius sweep: a fixed training size of 600 points (it was 300 in the radius sweep), the $\nu = 5/2$ Mat\'ern inner kernel, and 13 independent trials per grid point, with signal variance, bump radius, length scale, and noise variance re-optimized at each fixed value of $a$ by marginal-likelihood maximization, following Appendix~\ref{app:additional_info_on_teddy}. The sweep spans $a \in [0, 50]$, the same bound used during hyperparameter learning in the main experiment.

Figure~\ref{fig:amplitude_ablation_teddy_comparison} shows that the test metrics (RMSE, CRPS, PICP, and log-likelihood) are essentially flat over $a \in [0, 22]$: RMSE and CRPS sit at or near their best values, log-likelihood is at or near its maximum, and PICP is mildly below the nominal 0.95 level. Beyond $a \approx 22$, this stability breaks down: log-likelihood declines steadily for the remainder of the sweep, RMSE and CRPS both worsen substantially, and PICP drifts upward past nominal coverage toward mild over-confidence-in-reverse (over-coverage, $\approx 0.96$) by $a = 50$.

\begin{figure}[htbp]
    \centering
    \includegraphics[width=\textwidth]{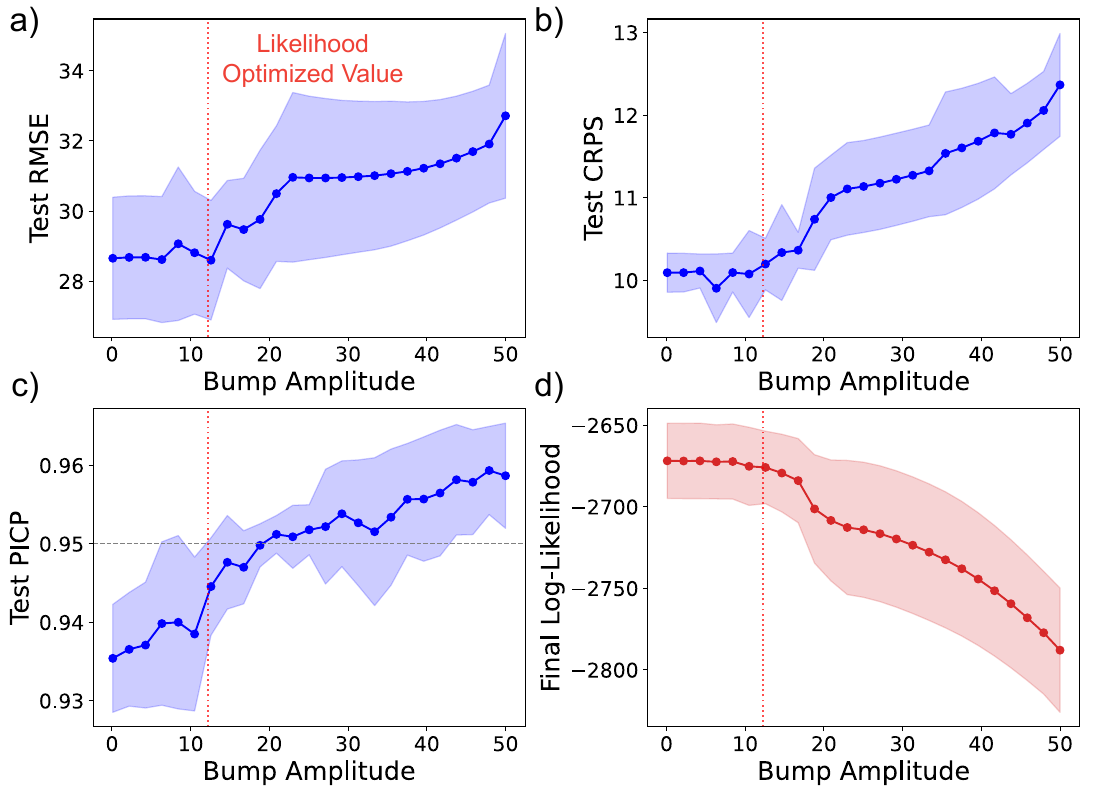}
    \caption{Sensitivity of the SLE kernel to the bump amplitude $a$ on the Teddy Bear manifold, at a fixed training size of 600 points with the $\nu = 5/2$ Mat\'ern inner kernel and 13 independent trials per grid point. At each grid point, $a$ is held fixed while the remaining hyperparameters (signal variance, bump radius, Mat\'ern length scale, and noise variance) are re-optimized by marginal-likelihood maximization, isolating the effect of $a$ after the model has been allowed to compensate through its other degrees of freedom. (a) Test RMSE, (b) Test CRPS, (c) Test PICP at the 95\% level (dashed horizontal line marks nominal coverage), and (d) the corresponding re-optimized log marginal likelihood, each as a function of $a$. The red dotted vertical line marks the amplitude selected by unconstrained marginal-likelihood maximization in the main Teddy Bear experiment (Appendix~\ref{app:additional_info_on_teddy}). Shaded bands denote one standard error across trials.}
    \label{fig:amplitude_ablation_teddy_comparison}
\end{figure}

Figure~\ref{fig:amplitude_discussion_ablation_teddy_comparison} makes explicit the mechanism underlying this pattern by tracking the re-optimized hyperparameters themselves rather than only the resulting predictive metrics. Over $a \in [0, 22]$, the flatness of the test metrics in Figure~\ref{fig:amplitude_ablation_teddy_comparison} is not because the underlying model is static — it is because two hyperparameters are actively compensating for the growing amplitude. Because $a$ rescales every nonzero entry of the embedding before the Euclidean distance $\|\phi(x) - \phi(x')\|$ is formed, increasing $a$ inflates typical inter-point distances in embedding space; the re-optimized Mat\'ern length scale $\ell$ and bump radius $r$ both increase steadily over this range (Figure~\ref{fig:amplitude_discussion_ablation_teddy_comparison}b–c) to offset this, keeping the effective kernel — and hence predictive performance — nearly unchanged. Signal variance (Figure~\ref{fig:amplitude_discussion_ablation_teddy_comparison}a), by contrast, remains comparatively flat and noisy over this same range, indicating it plays little role in the compensation while $\ell$ and $r$ still have room to grow.

\begin{figure}[htbp]
    \centering
    \includegraphics[width=\textwidth]{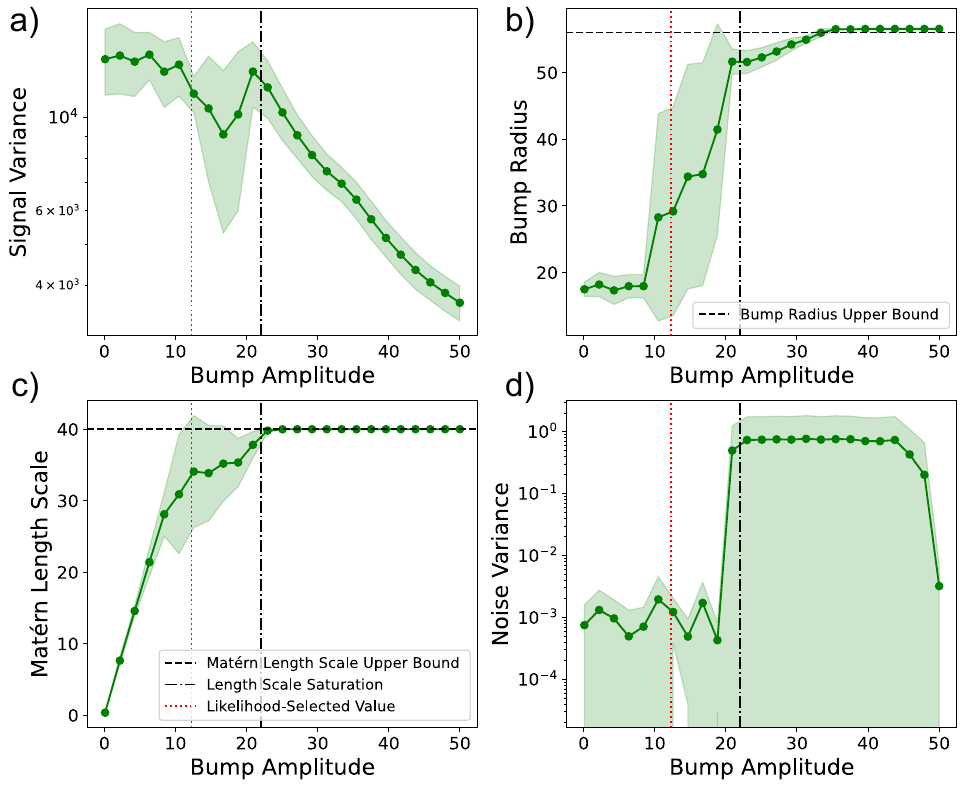}
    \caption{Re-optimized hyperparameters as a function of bump amplitude $a$, corresponding to the sweep in Figure~\ref{fig:amplitude_ablation_teddy_comparison}. At each grid point, signal variance (a), bump radius (b), Matérn length scale (c), and noise variance (d) are re-optimized by marginal-likelihood maximization while $a$ is held fixed. The dash-dotted black line marks the amplitude at which the length scale saturates at its upper bound ($a \approx 22$); the dotted red line marks the amplitude selected by unconstrained marginal-likelihood maximization in the main Teddy Bear experiment ($a=12.25$). Dashed black lines in (b) and (c) mark the configured upper bounds on bump radius (56) and length scale (40), respectively. Shaded bands denote one standard deviation across trials.}
    \label{fig:amplitude_discussion_ablation_teddy_comparison}
\end{figure}

This compensation is only possible while $\ell$ and $r$ have room to grow, and both reach their fixed upper bounds within the sweep. The length scale saturates first, reaching its configured upper bound of 40 at $a \approx 22$ (Figure~\ref{fig:amplitude_discussion_ablation_teddy_comparison}c); the radius continues increasing for a time afterward, reaching its own upper bound of 56 only around $a \approx 35$ (Figure~\ref{fig:amplitude_discussion_ablation_teddy_comparison}b). Once the length scale can no longer increase to track $a$, it is no longer sufficient on its own to keep the effective kernel unchanged, and the model falls back on two alternate mechanisms: signal variance begins declining steadily from that point onward (Figure~\ref{fig:amplitude_discussion_ablation_teddy_comparison}a), and noise variance jumps sharply, from a small, stable value below $10^{-2}$ to roughly 0.6–0.8 (Figure~\ref{fig:amplitude_discussion_ablation_teddy_comparison}d). This saturation-and-fallback sequence — length scale saturating first, radius following, then signal variance and noise variance absorbing the remainder — is the direct explanation for the divergence in RMSE, CRPS, and log-likelihood beyond $a \approx 22$ in Figure~\ref{fig:amplitude_ablation_teddy_comparison}, rather than any qualitative change in the kernel's locality. The sharp rise in noise variance, in particular, explains the mild PICP over-coverage observed in the same regime, since inflated noise variance directly widens predictive intervals, regardless of whether the underlying fit is improving. Noise variance drops again at $a = 50$, the extreme edge of the sweep range; since predictive performance is already substantially degraded throughout this saturated regime and this point simply marks the boundary of the search space rather than a qualitatively new operating condition, we do not interpret this final drop further.

The amplitude selected by unconstrained marginal-likelihood maximization in the main Teddy Bear experiment ($a \approx 12.25$, red dotted line in Figures~\ref{fig:amplitude_ablation_teddy_comparison} and ~\ref{fig:amplitude_discussion_ablation_teddy_comparison} falls well within the region where $\ell$ and $r$ can still freely compensate for $a$, comfortably below the saturation point at $a \approx 22$ where predictive performance begins to degrade. This also clarifies why fixing $a = 1$ in the Dragon experiment, rather than learning it, incurs no cost: at a value well inside this compensating region, amplitude, length scale, and radius trade off freely, and the model retains its full expressivity regardless of which specific value of $a$ is chosen within this regime.

\subsection{Inner Kernel}

For this sweep, the three most common choices of stationary kernel---Mat\'ern $\nu=5/2$, Mat\'ern $\nu=3/2$, and RBF — are each applied to the same bump embedding at a fixed training size of 300 points, with all remaining hyperparameters (signal variance, bump radius, bump amplitude, kernel length scale, and noise variance) independently re-optimized for each kernel by marginal-likelihood maximization, following the same protocol used elsewhere in Appendix~\ref{app:additional_info_on_teddy}. This isolates the effect of the inner kernel's functional form from the effect of the embedding itself, which is held fixed across all three configurations.

Figure~\ref{fig:inner_kernel_ablation_teddy_comparison} shows that predictive accuracy is essentially insensitive to the choice of inner kernel: the RMSE (a) and CRPS (b) distributions for all three kernels overlap substantially, with nearly identical medians and interquartile ranges, and no kernel is a clear or consistent winner across 15 trials. This is consistent with the sparse bump embedding doing the bulk of the representational work, with the inner kernel's functional form acting as a comparatively minor modulation on top of an already well-conditioned, locally structured input space.

Calibration and marginal likelihood, however, do show a modest but consistent separation between kernel choices that predictive accuracy alone does not reveal. Both Mat\'ern variants achieve median PICP closer to the nominal 0.95 level (panel c), with Mat\'ern $\nu=3/2$ slightly ahead of $\nu=5/2$; RBF, by contrast, shows both a lower median PICP ($\approx 0.92$) and a wider spread extending well below nominal coverage, indicating a mild but noticeable tendency toward overconfident intervals relative to the Mat\'ern kernels. This pattern is not mirrored in the log-likelihood panel (d): RBF achieves the highest (least negative) median log-likelihood of the three, with Mat\'ern $\nu=3/2$ the lowest, while Mat\'ern $\nu=5/2$ falls in between. In other words, the kernel most favored by the marginal-likelihood objective (RBF) is also the kernel with the weakest calibration on held-out data, while the best-calibrated kernel (Mat\'ern $\nu=3/2$) has the lowest training-time marginal likelihood of the three. This is a useful reminder that marginal-likelihood maximization selects for in-sample fit and is not a direct proxy for held-out calibration, and it provides a concrete, data-driven justification for the paper's choice to fix the inner kernel to the Mat\'ern family, matched to the kernel order of the corresponding baseline, rather than treating it as a free hyperparameter selected by likelihood alone.

Taken together, these results support a qualified version of the intended narrative for this section: the inner kernel's effect on point-prediction accuracy is negligible, consistent with the embedding---not the inner kernel--- driving predictive performance, but its effect on uncertainty calibration is real, if modest, and argues for the deliberate, baseline-matched choice of a Mat\'ern inner kernel used throughout the main experiments rather than for treating the inner kernel as an arbitrary or inconsequential design choice.

\begin{figure}[htbp]
    \centering
    \includegraphics[width=\textwidth]{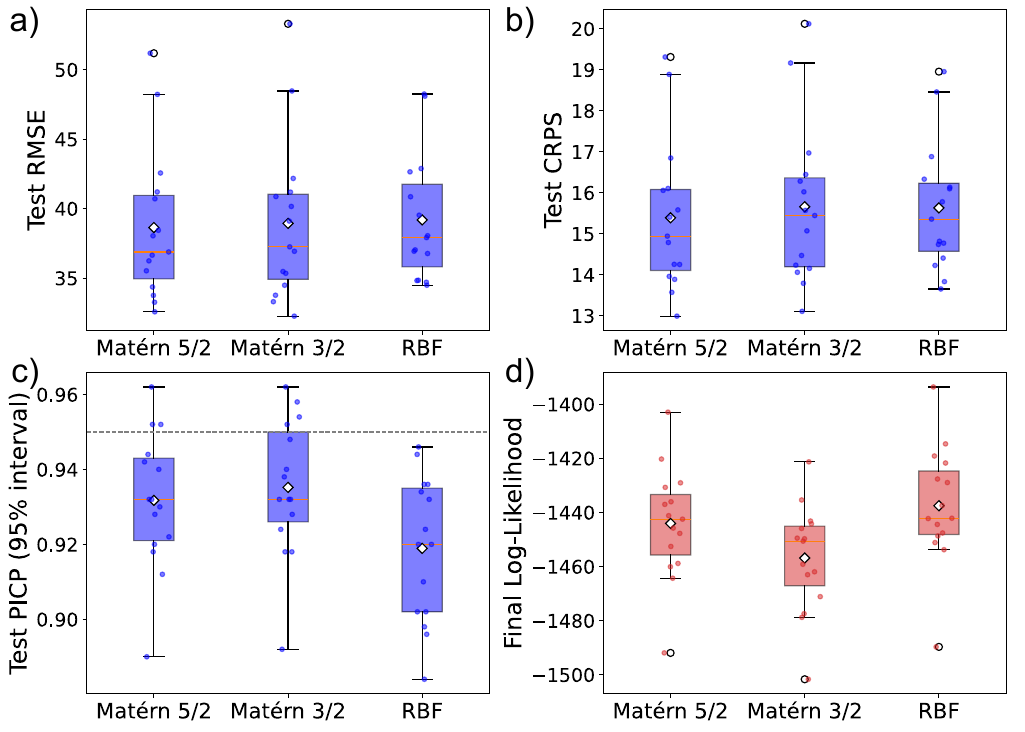}
    \caption{Sensitivity of the SLE kernel to the choice of inner kernel applied to the bump embedding, on the Teddy Bear manifold at a fixed training size of 300 points. For each of three inner kernel choices — Mat\'ern $\nu=5/2$, Mat\'ern $\nu=3/2$, and RBF — the remaining hyperparameters (signal variance, bump radius, bump amplitude, kernel length scale, and noise variance) are independently re-optimized by marginal-likelihood maximization, following the protocol described in Appendix~\ref{app:additional_info_on_teddy}. (a) Test RMSE, (b) Test CRPS, (c) Test PICP at the 95\% level (dashed horizontal line marks nominal coverage), and (d) the corresponding final log marginal likelihood. Box plots show the median (orange line), mean (white diamond), interquartile range (box), and full range excluding outliers (whiskers); individual trial values are overlaid as jittered points, with outliers outlined in black. Results are computed across 15 independent trials per kernel.}
    \label{fig:inner_kernel_ablation_teddy_comparison}
\end{figure}

\section{Empirical Conditioning of the Gram Matrix}\label{app:conditioning}
Theorem~\ref{thm:stability} predicts, under the bounded-overlap Assumption (A1),
that the SLE Gram matrix remains well-conditioned as the number of landmarks
$|\mathcal{D}|$ grows, in contrast to dense raw-distance embeddings. Here we
test this prediction directly. For the one-dimensional test function of
Figure~\ref{fig:comp}, Figure~\ref{fig:cond1d} shows that the condition number
is far better behaved for the SLE kernel --- growing only as
$\kappa_2 \sim N^{0.38}$ over $N \in [20, 1000]$, compared to
$\kappa_2 \sim N^{\approx 1}$ for the native (dense) distance-to-landmarks embedding. 
Both kernels are evaluated at the same
initial hyperparameters used in the runtime benchmark
(Table~\ref{tab:runtime_1d}), and both matrices are regularized by the
identical noise nugget $\sigma_n^2 = 0.01$ that the GP marginal-likelihood
solve actually uses, so the comparison isolates the effect of the embedding
itself rather than the regularizer. At small $N$, the two embeddings are
indistinguishable ($\kappa_{\text{native}}/\kappa_{\text{SLE}} \approx 0.9$ at
$N=20$), because the nugget floor dominates the smallest singular value on both
sides; as $N$ grows, the sparse-support geometry of the bump embedding caps the
effective feature dimension while the dense embedding's landmark coordinates
progressively concentrate, driving the ratio to
$\kappa_{\text{native}}/\kappa_{\text{SLE}} \approx 9.8$ at $N=1000$.

We note that this constitutes a conservative test of
Theorem~\ref{thm:stability}: in this benchmark the bump radius spans a constant
fraction of the domain, so the average overlap $\bar{s}$ grows linearly with
$N$ (Table~\ref{tab:runtime_1d}) and Assumption (A1) is deliberately
\emph{not} enforced --- yet conditioning still degrades dramatically more
slowly than for the dense embedding, and the gap opens exactly as the ambient
embedding dimension grows. This isolates dimensionality, rather than the
embedding per se, as the cause of the dense embedding's degradation,
consistent with the ablation results of Appendix~\ref{app:ablation}. Compact
support thus decouples the conditioning of $K$ from $|\mathcal{D}|$, keeping
SLE Gram matrices amenable to a numerically stable Cholesky factorization ---
and, consequently, to reliable posterior inference and hyperparameter
learning --- in regimes where the native embedding is already losing precision.

\begin{figure}
    \centering
    \includegraphics[width=0.8\linewidth]{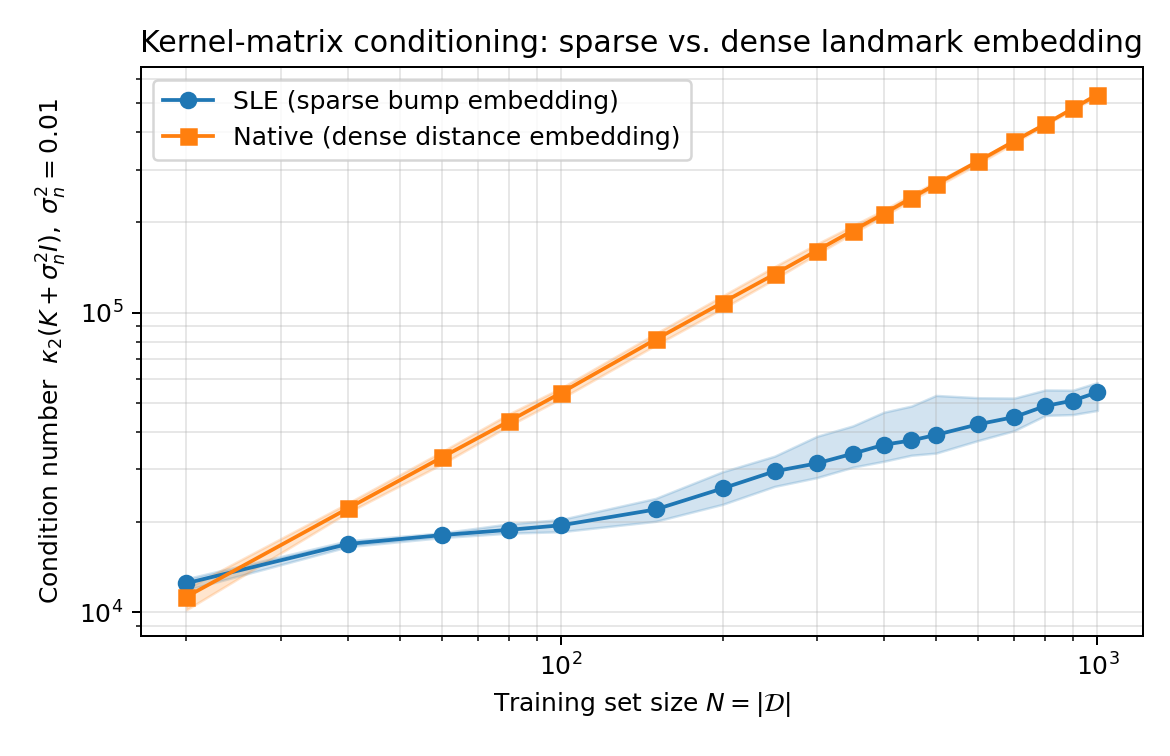}
    \caption{Gram-matrix condition number $\kappa_2(K + \sigma_n^2 I)$ with
    $\sigma_n^2 = 0.01$ for the SLE kernel versus the native (dense)
    distance-to-landmarks embedding on the 1D benchmark of
    Figure~\ref{fig:comp}. Lines are means and shaded bands min--max over 5
    random dataset draws; both kernels use the initial hyperparameters of
    Table~\ref{tab:runtime_1d} (no training). SLE conditioning grows only as
    $\kappa_2 \sim N^{0.38}$ while the native embedding grows as
    $\kappa_2 \sim N^{\approx 1}$ ($\approx 9.8\times$ worse at $N=1000$),
    consistent with Theorem~\ref{thm:stability}: compact bump support
    decouples the conditioning of $K$ from the ambient embedding dimension
    $|\mathcal{D}|$.}
    \label{fig:cond1d}
\end{figure}

\section{Empirical Distortion Analysis}\label{app:distortion}

Proposition~\ref{prop:injective} establishes that the SLE embedding is an injective function of the exact local distance profile, introducing no surrogate metric. Here, we complement this with an empirical comparison of embedding-space distances to native distances, alongside the sliced Wasserstein approximation. The analysis is performed on the SAXS distance matrices of Appendix~\ref{app:additional_info_saxs} using the embedding map directly, with all 500 images serving as landmarks; it characterizes the map itself and involves no fitted model.

Figure~\ref{fig:distortion}a plots the embedding distance $\lVert \phi(x) - \phi(x') \rVert$ against the true $W_2$ distance for all $124{,}750$ pairs, at $r = 0.18$, $\beta = 1$ and $a = 1$, giving a Spearman rank correlation of $\rho = 0.976$. Figure~\ref{fig:distortion}b repeats the comparison against the sliced Wasserstein distance ($\rho = 0.968$), and Figure~\ref{fig:distortion}c compares the sliced surrogate to the true $W_2$ directly ($\rho = 0.992$). The closeness of the values in (a) and (b) is a direct consequence of (c): since the sliced approximation is itself highly rank-correlated with the true $W_2$ distance, the SLE embedding's rank fidelity to one target is necessarily close to its rank fidelity to the other. Rank correlation is invariant to the bump amplitude, since $a$ rescales every embedding coordinate uniformly, but not to $\beta$, which is held at $1$ throughout.

\begin{figure}[t]
    \centering
    \includegraphics[width=\linewidth]{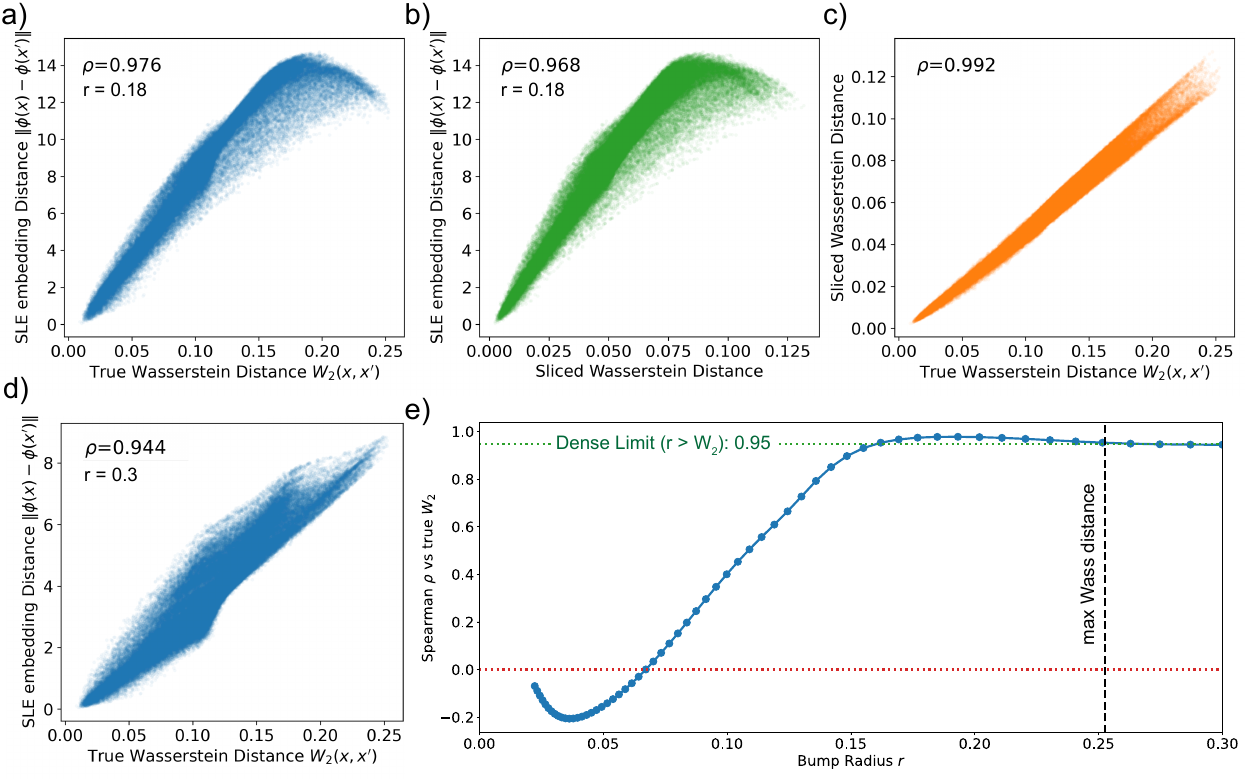}
    \caption{Empirical distortion analysis on the SAXS distance matrices, over all 124,750 pairs
    of the 500 images. (a) SLE embedding distance $\lVert\phi(x)-\phi(x')\rVert$ against the true
    Wasserstein-2 distance, at bump radius $r=0.18$, shape $\beta=1$, and amplitude $a=1$. (b) The
    same embedding distance against the sliced Wasserstein distance (50 projections). (c) Sliced
    Wasserstein against true $W_2$. (d) The same comparison as (a), at a larger bump radius
    $r=0.3$. (e) Rank correlation between the embedding distance and the true $W_2$ as a function
    of the bump radius $r$, over the same pair set. The green dotted line marks the globally
    supported (dense) limit obtained once $r$ exceeds the largest pairwise distance, and the black
    dashed line marks that distance. Spearman rank correlations are inset in panels (a)--(d). The
    analysis uses the embedding map applied directly to the precomputed distance matrices and is
    independent of any fitted GP model.}
  \label{fig:distortion}
\end{figure}

Within the bump reach, the embedding distance is a faithful increasing function of the native distance, with the scatter in Figure~\ref{fig:distortion}a being tight and the ordering essentially preserved. Beyond the reach, the relationship folds over, and this is a direct consequence of compact support rather than a distortion of the geometry. Once two inputs are separated by more than $r$, neither lies in the support of the other's bump, the coordinates that encode their mutual distance vanish, and the embedding distance reduces to the disjoint-support identity $\bigl( \lVert \phi(x) \rVert^{2} + \lVert \phi(x') \rVert^{2} \bigr)^{1/2}$ (Theorem~\ref{thm:concentration}). This residual quantity measures how densely each input's own neighborhood is populated rather than how far apart the two inputs are, and since the most widely separated pairs tend to lie in sparser regions, it decreases with $W_2$ over the far field. Far pairs are therefore no longer ordered by the embedding. This is precisely the far-field information that Proposition~\ref{prop:injective} states is discarded, and that Section~\ref{sec:discussion} records as a limitation of the deliberately local design; at $r = 0.18$ it concerns the 6.7\% of pairs separated by more than the radius.

Figure~\ref{fig:distortion}d repeats this comparison at a substantially larger radius, $r = 0.3$, larger than the largest pairwise distance in the dataset. At this radius every pair lies within reach of every bump, so the embedding is fully dense and the fold-over visible in Figure~\ref{fig:distortion}a is eliminated entirely: the scatter is monotonic across the full range of native distances. The resulting rank correlation, $\rho = 0.944$, is nonetheless slightly below the peak value obtained at $r = 0.18$, illustrating directly the shallow-maximum behavior quantified by the full sweep in Figure~\ref{fig:distortion}e: enlarging the radius removes the far-field fold-over but does not, on this dataset, improve rank fidelity beyond what is already achieved by a much sparser embedding.

Figure~\ref{fig:distortion}e shows that this behavior is not an artifact of the particular radius chosen. Rank fidelity rises steeply with $r$ and varies by less than $0.03$ for all $r \geq 0.15$, remaining high out to and beyond the largest pairwise distance in the dataset; the values shown in Figures~\ref{fig:distortion}a and~\ref{fig:distortion}d are drawn from this range. Two features of the curve are worth noting. First, the shallow maximum near $r \approx 0.18$ lies marginally above the dense limit reached once $r$ exceeds the data diameter and every bump is globally supported, so compact support costs nothing in rank fidelity relative to a dense embedding on this dataset while retaining exactly zero entries (Theorem~\ref{thm:sparsity}). Second, at very small radii, the correlation is mildly negative: below $r \approx 0.07$, almost no pair of inputs shares a landmark in common support, every embedding distance is governed by the density term above, and the ordering it induces runs weakly counter to the native one for the reason given in the preceding paragraph. This regime is far from any radius of practical interest and is shown only for completeness.

Taken together, the comparison confirms that the SLE embedding preserves the ordering of the native Wasserstein geometry within the region its bumps span, without introducing any projection or surrogate metric, and locates the boundary of that region exactly where Proposition~\ref{prop:injective} places it.

\section{Bump Function Visualization}\label{app:bump_vis}

Section~\ref{sec:methodology} introduces the bump function $b(d; a, r, \beta)$ of
Eq.~\eqref{eq:bump} (repeated here for reference)
\begin{equation}
b(d;\, a, r, \beta) =
\begin{cases}
a \, \exp\!\left(-\dfrac{\beta}{1 - d^{2}/r^{2}} + \beta\right), & d < r,\\[2ex]
0, & d \geq r,
\end{cases}
\label{eq:bump-app}
\end{equation}
as the map applied to each coordinate of the sparse landmark embedding
$\phi(x) = \big[b(d(x,x_1); a,r,\beta), \ldots, b(d(x,x_{|\mathcal{D}|}); a,r,\beta)\big]^{\top}$.
Figure~\ref{fig:bump-app} visualizes this function for fixed amplitude and support
radius ($a = 1$, $r = 1$) across three shape parameters,
$\beta \in \{0.5, 1, 5\}$, isolating the three properties on which the paper's
guarantees depend: compact support on $[0, r)$ (Theorems~\ref{thm:sparsity} and~\ref{thm:sparsity_scaling}, which give the sparsity of $\phi$ and hence of the
SLE embedding), $C^{\infty}$ smoothness on the support and at the boundary
$d = r$ (Theorem~\ref{thm:smooth}, which gives smoothness of the resulting
kernel), and strict positivity and monotonic decay on $[0, r)$ (the hypothesis
of Proposition~\ref{prop:injective}, which gives injectivity of $\phi$ with
respect to the local distance profile).

All three curves in Figure~\ref{fig:bump-app} are strictly positive and strictly
decreasing on $[0, 1)$ and drop to exactly zero at $d = r = 1$, with all
derivatives vanishing at the boundary rather than producing a kink --- this is
what allows a training point $x_i$ whose distance from $x$ satisfies
$d(x, x_i) \geq r$ to contribute \emph{exactly} zero to the embedding coordinate
$\phi_i(x)$, rather than a small but nonzero value, which is the mechanism
underlying the sparsity results of Section~\ref{sec:methodology}. The shape
parameter $\beta$ controls how the decay is distributed within the support. At
$\beta = 0.5$ the bump remains close to its peak value $a$ over most of $[0, r)$
and falls off steeply only as $d$ approaches the boundary, so that a landmark
contributes a nearly uniform weight until $x$ nears the edge of its support. At
$\beta = 5$, the bump decays rapidly from the origin and is already small well
before the boundary is reached, so that a landmark's contribution is sharply
concentrated on its immediate neighborhood, while the outer portion of the support contributes little. The $\beta = 1$ curve used throughout this work is intermediate between the two. Note that all three profiles share the same support radius $r$ and therefore the same sparsity pattern: $\beta$ changes the weighting \emph{within} the support, not which coordinates are nonzero. In the main text, $\beta$ is held fixed at $1$ and $r$ is learned by marginal-likelihood maximization with data-adaptive bounds (Appendices~\ref{app:more_info_dragon}--\ref{app:additional_info_saxs});
the sensitivity of predictive performance to $r$ and to the amplitude $a$ is reported separately in Appendix~\ref{app:sensitivity}.

\begin{figure}[t]
    \centering
    \includegraphics[width=0.6\linewidth]{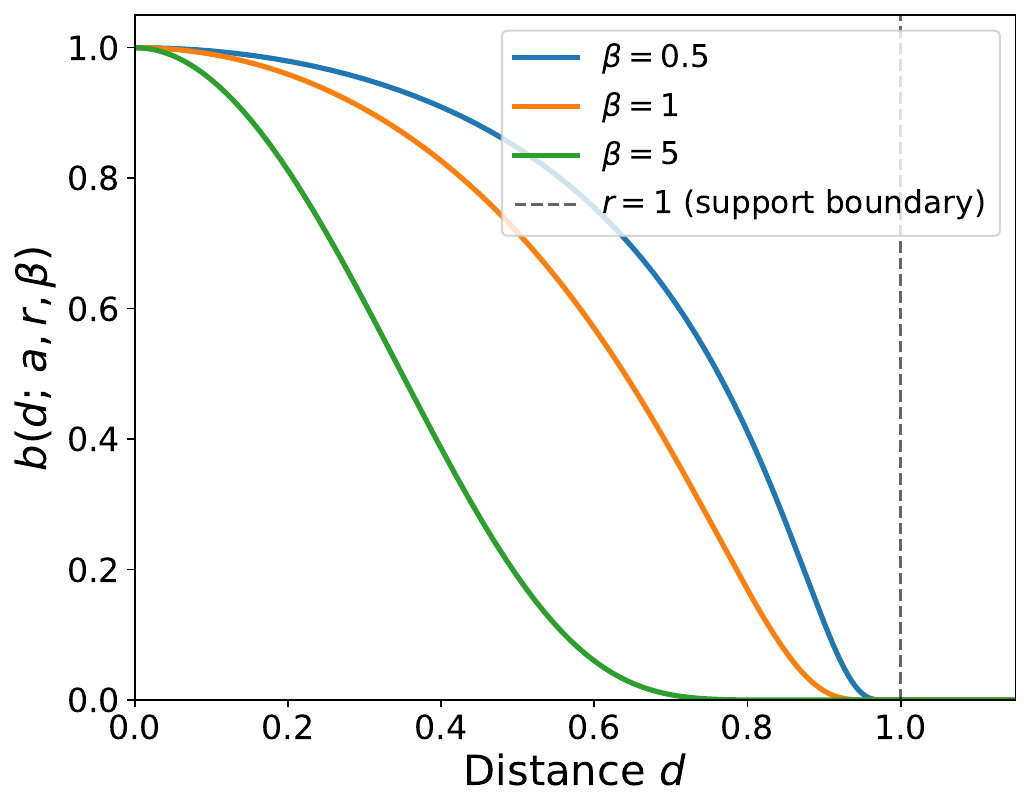}
    \caption{The bump function $b(d; a, r, \beta)$ of Eq.~\eqref{eq:bump-app}, plotted versus
    distance $d$ for fixed $a = 1$, $r = 1$, and three values of the shape parameter
    $\beta \in \{0.5, 1, 5\}$. The dashed vertical line marks the support boundary
    $d = r$, beyond which each embedding coordinate
    $\phi_i(x) = b(d(x,x_i); a,r,\beta)$ is identically zero. All curves attain the
    peak value $a$ at $d = 0$, are strictly positive and strictly decreasing on
    $[0, r)$, and vanish smoothly (all derivatives $\to 0$) at $d = r$, illustrating
    the compact support, $C^\infty$ smoothness, and strict monotonicity relied on by
    Theorems~\ref{thm:sparsity},~\ref{thm:sparsity_scaling}, and~\ref{thm:smooth} and Proposition~\ref{prop:injective}. The three curves differ only in how the
    decay is distributed across the support: larger $\beta$ decays more rapidly from
    the origin and lies below the smaller-$\beta$ curves at every $d \in (0, r)$,
    with the half-maximum crossing moving inward from $d \approx 0.76\,r$ at
    $\beta = 0.5$ to $d \approx 0.35\,r$ at $\beta = 5$.}
    \label{fig:bump-app}
\end{figure}
}

\rev{\section{Computational Cost}\label{app:cost}
Per kernel evaluation, the cost is $\mathcal{O}(s)$, where $s$ is the number of
overlapping nonzero embedding entries (Theorem~\ref{thm:scaling}), independent of
the ambient embedding dimension $|\mathcal{D}|$. The dominant cost is forming the
distance profiles between evaluation points and landmarks:
$\mathcal{O}(N\cdot|\mathcal{D}|)$ distance evaluations if computed naively. Two
observations put this cost in context. First, it is shared by every distance-based
competitor considered in this work: the sliced Wasserstein baseline computes the
same $N\cdot|\mathcal{D}|$ (sliced) distances, and the Riemannian and Geometric
kernels additionally require a Laplace--Beltrami eigendecomposition of the full
mesh. In all our experiments, the distance matrix is precomputed once and shared
across all methods and trials. Second, because only landmarks within radius $r$
contribute to the embedding, metric-space indexing structures --- cover trees,
vantage-point trees, or approximate nearest-neighbor search, which require only
the distance function and no coordinate representation --- reduce test-time
distance computation to range queries.

\paragraph{Memory and time relative to a standard distance-based kernel.}
Relative to a conventional distance-based kernel (e.g., a Mat\'ern kernel applied
directly to a CND distance), the SLE kernel requires \emph{no additional memory}:
both approaches consume the same $N \times |\mathcal{D}|$ pairwise distance
matrix and produce the same $N \times N$ Gram matrix. The sparse embedding adds
only $O(N\bar{s})$ nonzero entries, where $\bar{s}$ is the average number of
active bumps per point (Theorem~\ref{thm:scaling}), and need not be stored at
all, since each embedding row can be formed on the fly from the corresponding
distance row. In compute time, a plain distance kernel evaluates each entry from
a single precomputed distance in $O(1)$, whereas the SLE kernel compares two
sparse distance profiles in $O(s)$; since $s \ll |\mathcal{D}|$ and is
independent of $|\mathcal{D}|$, this overhead is a small constant factor rather
than a change in scaling, as the measured runtimes in Table~\ref{tab:runtime}
confirm. Spectral baselines (Riemannian, Geometric) store an $N \times l$
eigenvector matrix in place of a distance matrix, so the memory comparison there
is an equal trade instead of an overhead. Table~\ref{tab:runtime} reports
wall-clock training (including full hyperparameter optimization) and prediction
times, peak memory, and empirical sparsity $\bar{s}$ for SLE and all baselines
on a single CPU.

Table \ref{tab:runtime_1d} reports runtimes and RAM usage for the 1-dimensional synthetic function shown in Figure \ref{fig:comp}. The run was set up as follows. The SLE implementation builds a KD-tree over the landmark set once and reuses it across every likelihood evaluation of the MCMC sampler. For each query point, only landmark--query pairs with $d \le r$ are materialized (\texttt{cKDTree.sparse\_distance\_matrix}), the resulting embedding is stored as a CSR sparse matrix with $\bar{s}$ nonzeros per row, and the pairwise squared distance is formed with sparse matrix multiplication. The native kernel serves as a reference for the same landmark idea \emph{without} the bump; distances to all $|\mathcal{D}|$ landmarks are computed and stored densely. Kernel timings measure a single $K = k(X,X)$ evaluation at the initial hyperparameters; training timings measure a full MCMC hyperparameter run (200 samples, identical settings on all three kernels). All values are mean $\pm$ standard error over 3 random dataset draws on the same single-CPU host. \textbf{Empirical scaling ($y \sim N^{\alpha}$) fitted over $N \in [50, 500]$}: kernel evaluation $\alpha_{\text{SLE}}=1.75$, $\alpha_{\text{Mat\'ern}}=1.45$, $\alpha_{\text{native}}=2.51$; training $\alpha_{\text{SLE}}=2.16$, $\alpha_{\text{Mat\'ern}}=1.84$, $\alpha_{\text{native}}=2.60$. 

\textbf{Comparison with the manuscript.} Appendix~\ref{app:cost} predicts $\mathcal{O}(N \cdot |\mathcal{D}|)$ distance work and an $\mathcal{O}(N^2 \bar{s})$ Gram assembly for SLE versus $\mathcal{O}(N^2)$ for a raw-distance stationary kernel, i.e.\ a constant-factor overhead instead of a change in scaling order. Consistent with this, the SLE exponent exceeds that of the Mat\'ern baseline by $\Delta\alpha \approx 0.30$ for kernel evaluation and $\Delta\alpha \approx 0.32$ for training; both empirical exponents are below their asymptotic ideals of $2$ and $3$ because at $N\le 500$ Python and BLAS fixed overheads still contribute meaningfully to the timings. The native embedding, which materializes the full $N\times|\mathcal{D}|$ distance matrix without sparsification, tracks SLE closely ($\alpha_{\text{native}}=2.51$) because in this experiment $s$ is comparable to $|\mathcal{D}|$: the reported timings are taken at the initial radius $r=0.15$ rather than the MLE-selected radius, so the average sparsity grows as $\bar{s} \sim N^{1}$ ($\bar{s}\approx 13.9$ at $N=50$; $\bar{s}\approx 135.6$ at $N=500$). The manuscript's $\bar{s} = \mathcal{O}(1)$ regime requires $r$ to shrink with $N$, which occurs under marginal-likelihood maximization (Appendix~\ref{app:more_info_dragon}) but is not represented here; in that regime, the SLE kernel would separate more markedly from the native baseline. Peak $\Delta$RAM for kernel evaluation scales as $N^{0.96}$ (SLE), $N^{1.21}$ (Mat\'ern), and $N^{1.90}$ (native).

The native (dense) embedding baseline is fastest for small training sets — at $N=50$ its single-kernel evaluation is roughly $2.5\times$ faster than SLE — because it avoids the KD-tree query, CSR construction, and sparse-matmul overhead of the SLE implementation      
entirely. However, its kernel-evaluation wall-time scales as $N^{2.51}$ against SLE's $N^{1.75}$ and Mat\'ern's $N^{1.45}$, so by $N=500$ it becomes roughly $2\times$ slower than SLE and $4\times$ slower than Mat\'ern, with substantially larger run-to-run variance      
($107.5 \pm 58.8$~ms vs.\ $48.5 \pm 0.7$~ms for SLE); this crossover directly reflects the dense-embedding pathology that motivates the sparse SLE construction in the first place. The most decisive quantitative gap between the three kernels is memory: peak kernel-evaluation $\Delta$RAM 
scales as $N^{1.90}$ for the native embedding versus $N^{0.96}$ for SLE and $N^{1.21}$ for Mat\'ern, providing direct empirical support for the memory-scaling claim in this section.

\begin{table}[h]
\centering
\caption{Wall-clock runtimes (single CPU; distance/eigenpair precomputation
reported separately since it is shared or method-specific).}
\label{tab:runtime}
\begin{tabular}{llccc}
\toprule
Dataset & Model & Precompute & Train & Predict \\
\midrule
Dragon ($n{=}1000$)     & SLE                  & 12 h (geodesics, shared) & 3--4 min & 0.5 s \\
                        & Riemannian (500 ep)  & 2--3 min  & 2--3 min & 7.5 s \\
Teddy Bear ($n{=}400$)  & SLE                  & 1 h (geodesics, shared)  & 10.49 s & 0.0821 s \\
                        & Geometric Kernel     & 1 s  & 22.59 s & 0.11 s \\
SAXS ($n{=}250$)        & SLE -- Wass          & 1 h on 32 CPUs   & 62.73 s & 0.02 s \\
                        & Mat\'ern -- Sliced Wass & 1 h on 32 CPUs & 4.3 s & 0.01 s \\
\bottomrule
\end{tabular}
\end{table}
}

\rev{
\begin{table}[t]
\centering
\small
\setlength{\tabcolsep}{4pt}
\caption{Wall-clock time and peak resident memory for the proposed Sparse Landmark Embedding (SLE) kernel, a standard Mat\'ern ($\nu=3/2$) baseline, and the native (dense) distance-to-landmarks embedding kernel of Eq.~(3), on the 1D analytic benchmark of Figure \ref{fig:comp}. }
\label{tab:runtime_1d}
\begin{tabular}{cl cc cc c}
\toprule
 & & \multicolumn{2}{c}{\textbf{Kernel evaluation}} & \multicolumn{2}{c}{\textbf{Full training (MCMC)}} & \\
\cmidrule(lr){3-4}\cmidrule(lr){5-6}
$|\mathcal{D}|$ & Model & Time (ms) & $\Delta$RAM (MiB) & Time (s) & $\Delta$RAM (MiB) & $\bar{s}$ \\
\midrule
\multirow{3}{*}{50} & SLE & $1.19 \pm 0.14$ & $1.36 \pm 0.01$ & $0.14 \pm 0.01$ & $2.32 \pm 0.11$ & $13.9$ \\
                     & Mat\'ern & $1.10 \pm 0.01$ & $1.12 \pm 0.04$ & $0.13 \pm 0.01$ & $1.48 \pm 0.05$ & -- \\
                     & Native & $0.46 \pm 0.01$ & $0.14 \pm 0.04$ & $0.07 \pm 0.00$ & $2.43 \pm 0.05$ & -- \\
\cmidrule(lr){1-7}
\multirow{3}{*}{100} & SLE & $1.42 \pm 0.01$ & $1.56 \pm 0.05$ & $0.15 \pm 0.01$ & $2.66 \pm 0.03$ & $27.2$ \\
                     & Mat\'ern & $1.71 \pm 0.04$ & $1.24 \pm 0.01$ & $0.14 \pm 0.01$ & $1.26 \pm 0.28$ & -- \\
                     & Native & $0.83 \pm 0.02$ & $0.40 \pm 0.04$ & $0.10 \pm 0.00$ & $2.67 \pm 0.08$ & -- \\
\cmidrule(lr){1-7}
\multirow{3}{*}{150} & SLE & $2.15 \pm 0.03$ & $1.81 \pm 0.10$ & $1.56 \pm 0.16$ & $3.09 \pm 0.09$ & $40.8$ \\
                     & Mat\'ern & $2.77 \pm 0.01$ & $2.15 \pm 0.13$ & $1.39 \pm 0.18$ & $1.63 \pm 0.09$ & -- \\
                     & Native & $1.71 \pm 0.02$ & $1.10 \pm 0.05$ & $1.61 \pm 0.11$ & $2.92 \pm 0.04$ & -- \\
\cmidrule(lr){1-7}
\multirow{3}{*}{200} & SLE & $4.01 \pm 0.34$ & $3.34 \pm 0.04$ & $2.85 \pm 0.55$ & $4.08 \pm 0.42$ & $54.2$ \\
                     & Mat\'ern & $4.62 \pm 0.03$ & $3.63 \pm 0.06$ & $2.24 \pm 0.22$ & $2.82 \pm 0.60$ & -- \\
                     & Native & $4.41 \pm 0.33$ & $2.21 \pm 0.01$ & $2.45 \pm 0.29$ & $3.58 \pm 0.11$ & -- \\
\cmidrule(lr){1-7}
\multirow{3}{*}{250} & SLE & $5.34 \pm 0.07$ & $4.26 \pm 0.05$ & $3.09 \pm 0.32$ & $6.48 \pm 0.03$ & $67.5$ \\
                     & Mat\'ern & $6.91 \pm 0.67$ & $4.75 \pm 0.05$ & $2.13 \pm 0.12$ & $3.97 \pm 0.58$ & -- \\
                     & Native & $7.76 \pm 1.66$ & $3.05 \pm 0.03$ & $3.53 \pm 0.44$ & $6.69 \pm 0.05$ & -- \\
\cmidrule(lr){1-7}
\multirow{3}{*}{300} & SLE & $8.55 \pm 0.71$ & $5.00 \pm 0.21$ & $4.70 \pm 0.85$ & $8.26 \pm 0.20$ & $81.1$ \\
                     & Mat\'ern & $8.28 \pm 0.06$ & $6.10 \pm 0.04$ & $2.64 \pm 0.38$ & $5.97 \pm 0.02$ & -- \\
                     & Native & $10.10 \pm 0.20$ & $4.18 \pm 0.01$ & $4.77 \pm 0.29$ & $8.89 \pm 0.11$ & -- \\
\cmidrule(lr){1-7}
\multirow{3}{*}{350} & SLE & $20.17 \pm 0.13$ & $6.14 \pm 0.11$ & $5.08 \pm 0.65$ & $11.63 \pm 0.63$ & $94.6$ \\
                     & Mat\'ern & $13.22 \pm 0.13$ & $7.94 \pm 0.09$ & $3.55 \pm 0.12$ & $7.86 \pm 0.03$ & -- \\
                     & Native & $21.10 \pm 4.89$ & $5.31 \pm 0.04$ & $8.98 \pm 1.74$ & $12.35 \pm 0.06$ & -- \\
\cmidrule(lr){1-7}
\multirow{3}{*}{400} & SLE & $27.26 \pm 0.16$ & $7.53 \pm 0.05$ & $10.45 \pm 1.10$ & $15.11 \pm 0.06$ & $108.4$ \\
                     & Mat\'ern & $17.38 \pm 0.64$ & $10.50 \pm 0.10$ & $4.62 \pm 0.17$ & $9.63 \pm 0.02$ & -- \\
                     & Native & $57.68 \pm 33.94$ & $6.86 \pm 0.04$ & $10.43 \pm 2.38$ & $14.95 \pm 0.01$ & -- \\
\cmidrule(lr){1-7}
\multirow{3}{*}{450} & SLE & $37.06 \pm 0.38$ & $7.92 \pm 0.29$ & $11.13 \pm 0.93$ & $17.06 \pm 0.47$ & $122.2$ \\
                     & Mat\'ern & $22.09 \pm 1.17$ & $11.83 \pm 0.27$ & $5.27 \pm 0.84$ & $11.97 \pm 0.05$ & -- \\
                     & Native & $79.17 \pm 43.06$ & $8.37 \pm 0.01$ & $15.66 \pm 3.26$ & $18.20 \pm 0.03$ & -- \\
\cmidrule(lr){1-7}
\multirow{3}{*}{500} & SLE & $48.49 \pm 0.70$ & $10.87 \pm 0.10$ & $13.09 \pm 0.55$ & $20.64 \pm 0.10$ & $135.6$ \\
                     & Mat\'ern & $24.96 \pm 0.02$ & $13.88 \pm 0.20$ & $7.05 \pm 0.31$ & $14.55 \pm 0.08$ & -- \\
                     & Native & $107.48 \pm 58.81$ & $10.23 \pm 0.04$ & $16.82 \pm 3.53$ & $21.94 \pm 0.04$ & -- \\
\bottomrule
\end{tabular}
\end{table}
}

\newpage
\pagebreak
\section*{NeurIPS Paper Checklist}

\begin{enumerate}

\item {\bf Claims}
    \item[] Question: Do the main claims made in the abstract and introduction accurately reflect the paper's contributions and scope?
    \item[] Answer: \answerYes{}
    \item[] Justification: The abstract and introduction accurately describe the paper's contributions: a constructive framework for embedding non-Euclidean, non-conditionally negative definite (non-CND) input spaces into a reproducing kernel Hilbert space via bump-function landmark embeddings, with theoretical guarantees of positive definiteness and experimental validation on synthetic and real-world data.

\item {\bf Limitations}
    \item[] Question: Does the paper discuss the limitations of the work performed by the authors?
    \item[] Answer: \answerYes{}
    \item[] Justification: The Discussion/Conclusion section acknowledges key limitations.

\item {\bf Theory assumptions and proofs}
    \item[] Question: For each theoretical result, does the paper provide the full set of assumptions and a complete (and correct) proof?
    \item[] Answer: \answerYes{}
    \item[] Justification: All theorems and lemmas are numbered and state their assumptions explicitly. Full proofs are provided, either in the main text or in the appendix, with cross-references given in each theorem statement.

\item {\bf Experimental result reproducibility}
    \item[] Question: Does the paper fully disclose all the information needed to reproduce the main experimental results of the paper to the extent that it affects the main claims and/or conclusions of the paper (regardless of whether the code and data are provided or not)?
    \item[] Answer: \answerYes{}
    \item[] Justification: All experimental configurations, including kernel hyperparameters, landmark counts, data-generating processes, and evaluation metrics, are described in the experimental sections and appendix in sufficient detail to permit reproduction.

\item {\bf Open access to data and code}
    \item[] Question: Does the paper provide open access to the data and code, with sufficient instructions to faithfully reproduce the main experimental results, as described in the supplemental material?
    \item[] Answer: \answerYes{}
    \item[] Justification: A link to all code and datasets is provided with the final version, ensuring full reproducibility.

\item {\bf Experimental setting/details}
    \item[] Question: Does the paper specify all the training and test details (e.g., data splits, hyperparameters, how they were chosen, type of optimizer) necessary to understand the results?
    \item[] Answer: \answerYes{}
    \item[] Justification: Hyperparameter choices, optimization procedures, train/test splits, and evaluation metrics are specified in the experimental sections and supplemental appendix.

\item {\bf Experiment statistical significance}
    \item[] Question: Does the paper report error bars suitably and correctly defined or other appropriate information about the statistical significance of the experiments?
    \item[] Answer: \answerYes{}
    \item[] Justification: Results are reported with standard deviations computed over multiple independent runs with different random seeds; the number of runs and the source of variability are stated in the experimental sections.

\item {\bf Experiments compute resources}
    \item[] Question: For each experiment, does the paper provide sufficient information on the computer resources (type of compute workers, memory, time of execution) needed to reproduce the experiments?
    \item[] Answer: \answerYes{}
    \item[] Justification: The appendix reports the hardware used (CPU type, RAM) and approximate wall-clock runtimes for the main experiments. All experiments are computationally modest and reproducible on a standard workstation, with the exception of the geodesics calculations, which are not part of the proposed method.

\item {\bf Code of ethics}
    \item[] Question: Does the research conducted in the paper conform, in every respect, with the NeurIPS Code of Ethics?
    \item[] Answer: \answerYes{}
    \item[] Justification: The work is foundational kernel methodology research with no human subjects, sensitive data, or foreseeable direct pathways to harm; the authors have reviewed the NeurIPS Code of Ethics and confirm full conformance.

\item {\bf Broader impacts}
    \item[] Question: Does the paper discuss both potential positive societal impacts and negative societal impacts of the work performed?
    \item[] Answer: \answerNA{}
    \item[] Justification: The paper contributes foundational kernel theory for Gaussian processes on non-Euclidean spaces. There is no direct path to negative societal applications; positive impacts arise indirectly through improved uncertainty quantification in scientific machine learning.

\item {\bf Safeguards}
    \item[] Question: Does the paper describe safeguards that have been put in place for responsible release of data or models that have a high risk for misuse (e.g., pre-trained language models, image generators, or scraped datasets)?
    \item[] Answer: \answerNA{}
    \item[] Justification: The paper introduces a kernel construction method and does not release pre-trained models, scraped datasets, or any asset with high risk for misuse.

\item {\bf Licenses for existing assets}
    \item[] Question: Are the creators or original owners of assets (e.g., code, data, models), used in the paper, properly credited, and are the license and terms of use explicitly mentioned and properly respected?
    \item[] Answer: \answerYes{}
    \item[] Justification: All software dependencies (e.g., GP package, NumPy, SciPy) and referenced datasets are properly cited; open-source licenses are acknowledged where applicable. The full list will be included upon publication. 

\item {\bf New assets}
    \item[] Question: Are new assets introduced in the paper well documented and is the documentation provided alongside the assets?
    \item[] Answer: \answerNA{}
    \item[] Justification: No new standalone datasets or model checkpoints are released with this submission; the kernel implementation will be documented and released with the final version.

\item {\bf Crowdsourcing and research with human subjects}
    \item[] Question: For crowdsourcing experiments and research with human subjects, does the paper include the full text of instructions given to participants and screenshots, if applicable, as well as details about compensation (if any)?
    \item[] Answer: \answerNA{}
    \item[] Justification: The paper does not involve crowdsourcing or research with human subjects.

\item {\bf Institutional review board (IRB) approvals or equivalent for research with human subjects}
    \item[] Question: Does the paper describe potential risks incurred by study participants, whether such risks were disclosed to the subjects, and whether Institutional Review Board (IRB) approvals (or an equivalent approval/review based on the requirements of your country or institution) were obtained?
    \item[] Answer: \answerNA{}
    \item[] Justification: The paper does not involve human subjects and thus requires no IRB review.

\item {\bf Declaration of LLM usage}
    \item[] Question: Does the paper describe the usage of LLMs if it is an important, original, or non-standard component of the core methods in this research?
    \item[] Answer: \answerNA{}
    \item[] Justification: LLMs are not used as any component of the core methodology; the research concerns Gaussian process kernel construction on non-Euclidean spaces.

\end{enumerate}

\end{document}